\documentclass{article}

\usepackage{microtype}
\usepackage{graphicx}
\usepackage{subcaption}
\usepackage{booktabs} 

\usepackage{hyperref}

\usepackage[accepted]{icml2024}

\usepackage{amsmath}
\usepackage{amssymb}
\usepackage{mathtools}
\usepackage{amsthm}

\usepackage[capitalize,noabbrev]{cleveref}

\theoremstyle{plain}
\newtheorem{theorem}{Theorem}[section]
\newtheorem{proposition}[theorem]{Proposition}
\newtheorem{lemma}[theorem]{Lemma}
\newtheorem{corollary}[theorem]{Corollary}
\theoremstyle{definition}

\theoremstyle{remark}

\usepackage[textsize=tiny]{todonotes}

\def \bx {{\bm x}}
\def \by {{\bm y}}
\def \bz {{\bm z}}

\def \bh {{\bm h}}

\def \bW {{\mathbf W}}
\def \bsw {{\mathbf w}}

\def \bx {{\boldsymbol x}}
\def \by {{\boldsymbol y}}

\def \bH {{\mathbf H}}

\def \bw {{\mathbf W}}

\def \bwt {{\mathbf W}^{\top}}

\def \bsb {{\mathbf b}}
\def \bh {{\mathbf H}}
\def \barh {{\overline{\mathbf H}}}
\def \bsh {{\mathbf h}}

\def \bz {{\mathbf Z}}
\def \bsz {{\mathbf z}}

\usepackage{verbatim}
\usepackage{tablefootnote}

\usepackage{microtype}
\usepackage{graphicx}
\usepackage{booktabs} 
\usepackage{tablefootnote}
\usepackage{multirow}
\usepackage[font=small,skip=0pt]{caption}
\usepackage{enumitem}

\usepackage{siunitx}

\icmltitlerunning{Neural Collapse for Cross-entropy Class-Imbalanced Learning with Unconstrained ReLU Features Model}

\allowdisplaybreaks
\begin{document}

\twocolumn[
\icmltitle{Neural Collapse for Cross-entropy Class-Imbalanced Learning with Unconstrained ReLU Features Model}



\icmlsetsymbol{equal}{*}

\begin{icmlauthorlist}
\icmlauthor{Hien Dang}{austin,fpt}
\icmlauthor{Tho Tran}{fpt}
\icmlauthor{Tan Nguyen}{equal,nus}
\icmlauthor{Nhat Ho}{equal,austin}
\end{icmlauthorlist}

\icmlaffiliation{fpt}{FPT Software AI Center, Vietnam}
\icmlaffiliation{austin}{Department of Statistics and Data Sciences, University of Texas at Austin, USA}
\icmlaffiliation{nus}{Department of Mathematics, National University of Singapore, Singapore}

\icmlcorrespondingauthor{Hien Dang}{danghoanghien1123@gmail.com}

\icmlkeywords{Machine Learning, ICML}

\vskip 0.3in
]

\printAffiliationsAndNotice{\icmlEqualContribution} 




\begin{abstract}
 The current paradigm of training deep neural networks for classification tasks includes minimizing the empirical risk, pushing the training loss value towards zero even after the training classification error has vanished. In this terminal phase of training, it has been observed that the last-layer features collapse to their class-means and these class-means converge to the vertices of a simplex Equiangular Tight Frame (ETF). This phenomenon is termed as Neural Collapse ($\mathcal{NC}$). However, this characterization only holds in class-balanced datasets where every class has the same number of training samples. When the training dataset is class-imbalanced, some $\mathcal{NC}$ properties will no longer hold true, for example, the geometry of class-means will skew away from the simplex ETF. In this paper, we generalize $\mathcal{NC}$ to imbalanced regime for cross-entropy loss under the unconstrained ReLU features model. We demonstrate that while the within-class features collapse property still holds in this setting, the class-means will converge to a structure consisting of orthogonal vectors with lengths dependent on the number of training samples. Furthermore, we find that the classifier weights (i.e., the last-layer linear classifier) are aligned to the scaled and centered class-means, with scaling factors dependent on the number of training samples of each class. This generalizes $\mathcal{NC}$ in the class-balanced setting. We empirically validate our results through experiments on practical architectures and dataset.
\end{abstract}

\section{Introduction}
\label{sec:introduction}

Cross-entropy (CE) is undoubtedly one of the most popular loss functions used for training neural networks in the current deep learning paradigm. However, some crucial aspects of training networks using this loss function have not been fully explored yet, for example: i) Is there any unique pattern that the models learn when training deep neural networks until convergence, i.e., to reach zero loss?, ii) How do the learned network parameters vary across data distribution, training instances, and model architecture?, iii) What are the geometries of the representations and the classifier obtained from minimizing CE loss?. Understanding these questions is crucial for studying the training and generalization properties of deep neural networks. For instance, it has been a long-standing problem that training networks using CE loss under a long-tailed distribution dataset causes a significant drop in accuracy, especially for classes with a scarce amount of training samples. An important observation for this phenomenon is that the classifier weight vector of a more frequent class tends to have a larger norm, thus biasing the decision boundary toward the less frequent class. As a consequence, a smaller volume of the feature space is allocated for the minority classes, which leads to a drop in performance~\cite{kim20, Kang19, Cao19, ye20, liu23, kang20}.

A noticeable progress in answering these questions is the discovery of \textit{Neural Collapse} phenomenon \cite{papyan20}. Neural Collapse ($\mathcal{NC}$) reveals a common pattern of the learned last-layer features and the classifier weight of deep neural networks across canonical datasets and architectures. Specifically, $\mathcal{NC}$ consists of four properties emerging in the terminal phase of training of training deep neural networks for balanced datasets (i.e., every class has the same number of training instances):
\begin{itemize}
    \item $(\mathcal{NC}1)$ \textbf{Variability collapse:} features of the samples within the same class converge to a unique vector (i.e., the \textit{class-mean}), as training progresses.
    
    \item $(\mathcal{NC}2)$ \textbf{Convergence to simplex ETF:} the optimal class-means have the same length and are equally and maximally pairwise separated, i.e., they form a simplex Equiangular Tight Frame (ETF).
    
    \item $(\mathcal{NC}3)$ \textbf{Convergence to self-duality:} up to rescaling, the class-means and classifiers converge on each other.

    \item $(\mathcal{NC}4)$ \textbf{Simplification to nearest class-center:} given a feature, the classifier converges to choosing whichever class has the nearest class-mean to it.
\end{itemize}

The intriguing empirical observation of Neural Collapse has attracted many theoretical investigations, mostly under a simplified \textit{unconstrained features model (UFM)} (see Section \ref{sec:problemsetup} for more details) and for \textit{class-balanced dataset}, demonstrating that $\mathcal{NC}$ properties occur at any global solution of the loss function. However, under \textit{class-imbalanced dataset}, it has been observed that deep neural networks exhibit different geometric structures and 
some $\mathcal{NC}$ properties are not satisfied anymore \cite{dang23, Christos22, hong23}. The last-layer features of samples within the same class still converge to their class-means ($\mathcal{NC}1$), but the class-means and the classifier weights will no longer form a simplex ETF ($\mathcal{NC}2$) \cite{Fang21}. In a more extreme case where the imbalance level exceeds a certain threshold, the learned classifiers of the minority classes collapse onto each other, becoming indistinguishable from those of other classes \cite{Fang21}. This phenomenon, known as \textit{Minority Collapse}, explains why the accuracy for these minority classes drops significantly compared to the class-balanced setting. 

While the Neural Collapse emergence at the optimal solution in deep neural networks training using CE loss for \textit{balanced dataset} has been extensively studied \cite{Lu20, Zhu21}, the corresponding characterization for this loss function in \textit{imbalanced scenario} 
has remained limited. Under imbalanced regime, several theoretical works have characterized Neural Collapse phenomenon for other loss functions. In particular, \cite{dang23} has demonstrated the convergence geometry of the learned features and learned classifier for the mean squared error (MSE) loss. \cite{Christos22} studies the support vector machine (SVM) problem, whose global
minima follows a different geometry known as Simplex-Encoded-Labels Interpolation (SELI) and later \cite{Behnia23} extends it to some other SVM parameterizations.

\noindent \textbf{Comparison to concurrent work \cite{hong23}:} For CE loss, we acknowledge that the concurrent work \cite{hong23} is closely related to our work. They investigate Neural Collapse for CE loss with UFM under imbalanced setting. They prove the within-class features collapse property ($\mathcal{NC}1$) and demonstrate the \textit{network output prediction vectors converge to a block structure} where each block corresponds to classes that have the same amount of training samples. However, their analysis does not cover the magnitude of prediction vectors and the magnitude of each block within the structure. Consequently, it is not yet possible to describe the geometry explicitly and quantify how the structure changes under different imbalance levels. Additionally, the geometry of the learned last-layer feature, the classifier weight $(\mathcal{NC}2)$ and the relationship between them $(\mathcal{NC}3)$ have not been examined in \cite{hong23}. As a result, the corresponding $(\mathcal{NC}2)$ and $(\mathcal{NC}3)$ properties has not been characterized for this setting. Moreover, other considerations such as the norm, the norm ratio, and angle between the classifier weights and features are also not addressed. 

On the other hand, in this paper, we study CE loss training problem using UFM, but with a different setting, in which the features have to be \textit{element-wise non-negative}. This setting is motivated by the current paradigm that features are typically the outcome of some non-negative activation function, like ReLU or sigmoid. In this setting, we study the global solutions of CE training problem under UFM and present a thorough analysis of the convergence geometry of the last-layer features and classifier. We summarize our contributions as follows:
\begin{itemize}
    \item We explicitly characterize Neural Collapse for the last-layer features and classifier weights in CE loss training with non-negative features in class-imbalanced settings. We prove that at optimality, $\mathcal{NC}1$ still occurs, and the optimal class-means form an orthogonal structure. We derive the closed-form lengths of the features, in terms of the number of training samples and other hyperparameters.
    
    \item We find that the classifier weight is aligned to a scaled and centered version of the class-means, which generalizes the properties $\mathcal{NC}2$ and $\mathcal{NC}3$ from the original definition of Neural Collapse. Additionally, we derive the norms, norm ratios and angles between these classifier weights explicitly.
    
    \item We derive the exact threshold of the amount of training samples for a class to collapse and become indistinguishable from other classes. Hence, the threshold for Minority Collapse is also obtained in our analysis.

\end{itemize}     

\noindent \textbf{Notation:} For a weight matrix $\mathbf{W}$, we use $\mathbf{w}_{j}$ to denote its $j$-th row vector. $\| . \|_{F}$ denotes the Frobenius norm of a matrix and $\| . \|_{2}$ denotes $\text{L}_{2}$-norm of a vector.  $\otimes$ denotes the Kronecker product. The symbol ``$\propto$'' denotes proportional, i.e, equal up to a positive scalar. We also use some common matrix notations: $\mathbf{1}_{n}$ is the all-ones vector,   $\operatorname{diag}\{a_{1},\ldots, a_{K}\}$ is a square diagonal matrix size $K \times K$ with diagonal entries $a_{1},\ldots, a_{K}$. We use $[K]$ to denote the index set $\{1, 2, \ldots, K \}$.

\section{Related Works}
\label{sec:rw}

\textbf{Neural Collapse on balanced dataset:} A surge of theoretical results for $\mathcal{NC}$ under balanced scenario has emerged after the discovery of this phenomenon. Due to the highly non-convexity of the problem of training deep networks, theoretical works have proven the occurrence of $\mathcal{NC}$ for different loss functions and architectures with a simplified unconstrained features model (see Section \ref{sec:problemsetup} for more details) \cite{Lu20, Zhu21, graf23, Zhou22, Zhou22b, Tirer22, dang23, Christos22, Behnia23, kini23}. In particular,  $\mathcal{NC}$ properties are proven to occur at the optimal last-layer features and classifier across different loss functions: cross-entropy \cite{Lu20, Zhu21}, mean squared error \cite{Zhou22, Tirer22, dang23},  supervised contrastive loss \cite{graf23} and also for focal loss and label smoothing \cite{Zhou22b}. Recent works have spent efforts to extend the UFM to deeper architectures to study the behavior of more layers after the "unconstrained features". Specifically, \cite{Tirer22} extends UFM to account for one additional layer, from one-layer linear classifier to two-layer linear classifier after the ”unconstrained” features for MSE loss, and later the work \cite{dang23} extends the setting to a general deep linear network for both MSE and CE losses. \cite{Tirer22} also extends UFM for MSE loss to a two-layer case with ReLU activation. This setting is later extended by \cite{sukenik23} to the general deep UFM with ReLU activation for the binary classification problem. For multiclass classification problem with MSE loss, recent extensions to account for additional layers in the analysis with non-linearity are studied in \cite{Tirer22, Rangamani22}, or with batch normalization \cite{Ergen20}. However, these works require strong assumptions on the global optimal solution or the network architecture and capability for their theoretical results to be hold. There are also efforts to mitigate the restriction of UFM, such as \cite{Tirer23} analyzes UFM with an additional regularization term to force the features to stay in
the vicinity of a predefined feature matrix (e.g., intermediate features). Additionally, \cite{Zhu21, Zhou22, Zhou22b} prove the benign optimization landscape for several loss functions under UFM, demonstrating that critical points can only be global minima or strict saddle points. 

\noindent \textbf{Neural Collapse on imbalanced dataset:} The work \cite{Fang21} is likely the first to observe that for imbalanced setting, the collapse of features within the same class $\mathcal{NC}1$ is preserved, but the geometry skews away from the ETF. They also present the "Minority Collapse" phenomenon, in which the minority classifiers collapse to the same vector if the imbalance level is greater than some threshold. For MSE loss, \cite{dang23} has explicitly characterized the geometry of the learned features and classifiers for imbalanced setting. \cite{dang23} showed that the $\mathcal{NC}1$ still holds and the class-means converge to a General Orthonormal Frame (GOF), which consists of orthonormal vectors but with different lengths. By applying non-negative constraints for the normalized features to incorporate the effect of ReLU activation, \cite{kini23} finds the global minimizers of supervised contrastive loss and proves that the optimal features form an Orthogonal Frame (OF) with equal length and orthogonal vectors, regardless of the imbalance level.
\cite{Christos22} theoretically studies the UFM-SVM problem, whose global minima follow a more general geometry than the ETF, called "SELI". However, this work also makes clear that the unregularized version of CE loss only converges to KKT points of the SVM problem, which are not necessarily global minima. The result for UFM-SVM is later extended by the work \cite{Behnia23} to consider several cross-entropy parameterizations.  

Regarding CE loss, \cite{Yang22} studies the imbalanced setting but with fixed, unlearnable last-layer linear classifiers as a simplex ETF. They prove that no matter whether the data distribution is balanced or not among classes, the features will converge to a simplex ETF in the same direction as the fixed classifier. As mentioned in Section \ref{sec:introduction}, the work \cite{hong23} is closely related to our work and they study CE loss with UFM and the features can have negative entries. They prove that at optimality, the within-class features collapse ($\mathcal{NC}1$) and the \textit{network output prediction vectors converge to a block structure}. However, their analysis does not cover the magnitude of prediction vectors and the ratio between each block within the structure are not yet covered. Thus, it is not possible to describe the geometry explicitly and quantify how the structure changes under different imbalance levels. Additionally, the structure of the learned last-layer feature, the classifier weight $(\mathcal{NC}2)$ and the relation between them $(\mathcal{NC}3)$ have not been derived in \cite{hong23}. As a result, the corresponding $(\mathcal{NC}2)$ and $(\mathcal{NC}3)$ properties has not been characterized for this setting.


\section{Problem Setup}
\label{sec:problemsetup}
\textbf{Training Neural Network with Cross-Entropy Loss}:
In this work, we focus on neural network trained using the cross-entropy (CE) loss function on an imbalanced dataset. We consider the classification task with $K$ classes. Let $n_{k}$ denote the number of training samples in class $k \in [K]$, and $N:= \sum_{k=1}^{K} n_{k}$, the total number of training samples. A typical deep neural network classifier consists of a feature mapping function $\bsh(\boldsymbol{x})$ and a linear classifier parameterized as $\bw$. Specifically, a typical $L$-layer deep neural network can be expressed as follows:  
\begin{align}
    \psi_{\bw, \boldsymbol{\theta}}(\bx) 
    = \bw_{L}
    \underbrace{
    \sigma(\mathbf{W}_{L - 1} \ldots
    \sigma(\mathbf{W}_{1} \boldsymbol{x} + \mathbf{b}_{1}) + \mathbf{b}_{L - 1} )}
    _{\text{Feature } 
    \bsh = \bsh(x)}, \nonumber
\end{align}
where each layer composes of an affine transformation parameterized by a weight matrix $\bw_{l}$ and bias $\bsb_{l}$, followed by a non-linear activation $\sigma$ (e.g., $\text{ReLU}(x) = \max(x, 0)$). Here, $\boldsymbol{\theta} :=\{\mathbf{W}_{l}, \mathbf{b}_{l}\}_{l=1}^{L - 1}$ is the set of all learnable parameters in the feature mapping. We denote the last layer linear classifier as $\bw := \bw_{L}$ for convenience.
Current paradigm trains the network by minimizing the empirical risk over all training samples $\{ (\bx_{k,i}, \by_{k}) \}_{k, i}$ where $\bx_{k,i}$ denoted the $i$-th sample of class $k$ and $\by_{k}$ is the one-hot label vector for class $k$:
\begin{align}
    \min_{ \bw, \boldsymbol{\theta}}
    \frac{1}{N}
    \sum_{k=1}^{K} \sum_{i = 1}^{n_k}
    \mathcal{L}
    (\psi_{\bw, \theta}(\bx_{k,i}), \by_{k})
    + \frac{\lambda_{W}}{2} \| \bw \|_F^2
    + \frac{\lambda_{\theta}}{2} \| \boldsymbol{\theta} \|^2, \nonumber
\end{align}
where $\lambda_{W}, \lambda_{\theta} > 0$ are the weight decay parameters and $\mathcal{L}(\psi(\mathbf{x}_{k,i}), \mathbf{y}_{k})$ is the loss function that measures the difference between the output $\psi(\mathbf{x}_{k,i})$ and the target $\mathbf{y}_{k}$. For a vector $\mathbf{z} = [z_1, z_2, \ldots, z_K] \in \mathbb{R}^{K}$ and a target one-hot vector $\by_{k}$, CE loss is defined as:
\begin{align}
    \mathcal{L}_{CE} (\mathbf{z}, \by_{k})
    = - \log 
    \left(
    \frac{\exp(z_k)}{\sum_{m = 1}^{K} \exp(z_j)}
    \right)
\end{align}

\noindent \textbf{Unconstrained Features Model (UFM) with non-negative features:} Due to the significant challenges of analyzing the highly non-convex neural network training problem, recent theoretical works study $\mathcal{NC}$ phenomenon using a simplified model called unconstrained features model (UFM), or, layer-peeled model \cite{Fang21}. In particular, UFM peels down the last-layer of the network and treats the last-layer features $\mathbf{h}_{k,i}= \bsh(\mathbf{x}_{k,i}) \in \mathbb{R}^{d}$ as free optimization variables in order to capture the main characteristics of the last layers related to $\mathcal{NC}$ during training. This relaxation can be justified by the well-known result that an overparameterized deep neural network can  approximate any continuous function \cite{Hornik89, Hornik91, Zhou18, Yarotsky18}. 

In this work, we consider a slight variant of UFM, in which the \textit{features are constrained to be non-negative}, motivated by the fact that features are usually the output of ReLU activations in many common architectures. 
Formally, we consider the following modified version of UFM trained with CE loss with non-negative features:
\begin{align}
    \label{eq:UFM_plus}
    &\min_{\bw, \bh}
    \frac{1}{N}
    \sum_{k=1}^{K} \sum_{i = 1}^{n_k}
    \mathcal{L}_{CE}
    ( \bw \bsh_{k,i} , \by_{k})
    + \frac{\lambda_{W}}{2} \| \bw \|_F^2
    \\
    &+ \frac{\lambda_{H}}{2} \| \bh \|_F^2,
    \quad \text{s.t.} \quad \bh \geq 0, \lambda_{W} > 0, \lambda_{H} > 0,
    \nonumber
\end{align}
where $\mathbf{H}:=[\mathbf{h}_{1,1},\ldots, \mathbf{h}_{1,n_{1}}, \mathbf{h}_{2,1},\ldots, \mathbf{h}_{K,n_{K}}] \in \mathbb{R}^{d \times N}$ and $\bh \geq 0$ denotes entry-wise non-negativity. We note that similar settings with ReLU features were previously considered in \cite{nguyen22}, where $\mathcal{NC}$ configuration was derived for the label smoothing loss under balanced setting, and in \cite{kini23}, which studied the convergence geometry for supervised contrastive loss under imbalanced setting. We denote this setting as $\text{UFM}_{+}$, as \cite{kini23}, to differentiate it from the original UFM. 

By denoting $\bw = [\bsw_{1}, \bsw_{2}, \ldots, \bsw_{K}]^{\top} \in \mathbb{R}^{K \times d}$ be the last-layer weight matrix, with $\bsw_{k} \in \mathbb{R}^{d}$ is the $k$-th row of $\bw$, the CE loss can be written as:
\begin{align}
    \mathcal{L}_{CE}(\bw \bsh_{k,i}, \mathbf{y}_{k})
    = - \log 
    \left(
    \frac{\exp(\bsw_{k}^{\top} \bsh_{k,i})}{\sum_{m = 1}^{K} \exp( \bsw_{m}^{\top} \bsh_{k,i})} 
    \right). \nonumber
\end{align}

We also denote the \textit{class-mean} of a class $k \in [K]$ as $\bsh_k := n_k^{-1} \sum_{i=1}^{n_k} \bsh_{k,i}$ and the \textit{global-mean} $\bsh_{G} := N^{-1} \sum_{k=1}^{K} \sum_{i=1}^{n_k} \bsh_{k,i}$. The \textit{class-mean matrix} is denoted as $\barh = [\bsh_1, \bsh_2, \ldots, \bsh_K] \in \mathbb{R}^{d \times K}$.

\noindent \textbf{Neural Collapse for balanced dataset:} With the notations defined above, we recall the $\mathcal{NC}$ properties in the balanced setting as follows:
\begin{itemize}
    \item ($\mathcal{NC}1$) \textbf{Variability collapse}: 
    \begin{align}
       \bsh_{k,i} = \bsh_{k}, \quad \forall k \in [K], i \in [n_k]. \nonumber
    \end{align}
    \item ($\mathcal{NC}2$) \textbf{Convergence to simplex ETF}: 
    \begin{align}
        (\barh - \bsh_{G} \mathbf{1}_{K}^{\top})^{\top} (\barh - \bsh_{G} \mathbf{1}_{K}^{\top}) \propto \mathbf{I}_{K} - \frac{1}{K} \mathbf{1}_{K} \mathbf{1}_{K}^{\top}. \nonumber
    \end{align}
    \item ($\mathcal{NC}3$) \textbf{Convergence to self-duality}: 
    \begin{align}   
    \bw \propto  (\barh - \bsh_{G} \mathbf{1}_{K}^{\top})^{\top}. \nonumber        
    \end{align}
\end{itemize}

\noindent \textbf{Orthogonal Frame and General Orthogonal Frame:} Some previous results, such as \cite{Tirer22, nguyen22}, derive that under the balanced setting, the optimal class-means $\{ \bsh_k \}$ form an \textit{orthogonal frame (OF)}, i.e., $\barh^{\top} \barh \propto \mathbf{I}_{K}$. By centering the OF structure with its mean vector, we will receive a simplex ETF. Thus, this structure still follows $(\mathcal{NC}2)$ property. For MSE loss under class-imbalanced scenario, \cite{dang23} proves that the class-means $\{ \bsh_k \}$ form an orthogonal structure consisting of pairwise orthogonal vectors but having different lengths. They termed this structure as \textit{general orthogonal frame} (GOF). We will use this notation for our results in Section \ref{sec:geometry}.

\textbf{SELI geometry:} Simplex-Encoded-Labels Interpolation (SELI) is the geometric structure of the optimal classifier, feature and the prediction matrix of the imbalanced SVM training problem under UFM \cite{Christos22}. In particular, the prediction matrix $\bz = \bw \bh \in \mathbb{R}^{K \times N}$ is proven to have its $i$-th column to be $\by_i - \frac{1}{K} \mathbf{1}_{K}$ for all $i \in [N]$. Then, if we denote the SVD of the matrix $\mathbf{Z} = \mathbf{V} \Lambda \mathbf{U}^{\top}$, the classifier and feature matrices satisfy that $\bw \bwt = \mathbf{V} \Lambda \mathbf{V}^{\top}$ and $\mathbf{H}^{\top} \mathbf{H} = \mathbf{U} \Lambda \mathbf{U}^{\top}$. For UFM-CE training problem but without regularization ($\lambda = 0$), \cite{Ji21} showed that its gradient flow converges in direction to a Karush-Kuhn-Tucker (KKT) point of the UFM-SVM problem. Thus, SELI geometry is not necessarily the global minima of the unregularized UFM-CE problem.

\section{Main Result: Global Structure of $\text{UFM}_{+}$ Cross-Entropy Imbalanced}
\label{sec:geometry}

In this section, we characterize the global solution $(\bw, \bh)$ of the non-convex problem \eqref{eq:UFM_plus} and analyze its geometries. We prove that irrespective of the label distribution, the optimal features form an orthogonal structure in the non-negative orthant while the classifiers align with the scaled-and-centered features and spread across the entire feature space with $\sum_{k=1}^{K} \bsw_{k} = \mathbf{0}$. For convenience, we define the following constants for every class $k \in [K]$:
\begin{align}
    &\overline{M}_k := \log \left( (K - 1) \left(
    \frac{\sqrt{n_k}}{N
    \sqrt{\frac{K-1}{K} \lambda_{W} \lambda_{H}}}
    - 1\right) \right), \nonumber \\
    &M_k := \left\{\begin{matrix}
    \overline{M}_k &&\text{if } \overline{M}_k > 0 \\ 
    0 &&\text{if } \overline{M}_k \leq 0 \text{ or } \overline{M}_k \text{ is undefined}
    \end{matrix}\right.
    .
    \label{eq:M_k}
\end{align}
Note that the inequality $M_k = \overline{M}_k > 0$ is equivalent to $\frac{N}{\sqrt{n_k}} \sqrt{ \lambda_{W} \lambda_{H}} < \sqrt{\frac{K-1}{K}}$ and $M_k = 0$ when and only when $\frac{N}{\sqrt{n_k}} \sqrt{ \lambda_{W} \lambda_{H}} \geq \sqrt{\frac{K-1}{K}}$ . We state the our main result in the following theorem.

\begin{theorem}[Geometry of $\text{UFM}_{+}$ Cross-Entropy Imbalanced minimizers]
\label{thm:CE_main}
Suppose $d \geq K$ and $\frac{N}{\sqrt{n_k}} 
    \sqrt{ \lambda_{W} \lambda_{H}} < \sqrt{\frac{K-1}{K}} \: \forall \: k \in [K]$, then any global minimizer $(\bw, \bh)$ of the problem \eqref{eq:UFM_plus} obeys

\begin{enumerate}[label=(\alph*)]
    \item Within-class feature collapse:
    \begin{align}
    \label{eq:NC_1}
     \forall \: k \in [K], \: \bsh_{k,i} = \bsh_{k,j}, \quad \forall \: i \neq j. 
    \end{align}

    \item Class-mean orthogonality:
    \begin{align}
    \label{eq:feature_ortho}
    \bsh_{k}^{\top} \bsh_{l} = 0, \quad \forall \: k \neq l. 
    \end{align}

    \item Class-mean norm:
    \begin{align}
    \label{eq:feature_norm}
    \| \bsh_{k} \|^2 =
      \sqrt{\frac{K-1}{K} \frac{\lambda_{W}}{ \lambda_{H}} \frac{1}{n_k}}
      M_k. 
    \end{align}

    \item Relation between the classifier and class-means:
    \begin{align}
    \label{eq:NC2}
    &\bsw_{k}
    = \sqrt{\frac{\lambda_{H}}{\lambda_{W} K (K-1)}}
     \left( K \sqrt{n_k} \bsh_{k} - \sum_{m=1}^{K} \sqrt{n_m} \bsh_{m}
     \right)
     \\
    &\text{and } \sum_{k=1}^{K} \bsw_{k} = \mathbf{0} \nonumber.
    \end{align}

    \item Prediction vector of class $k$-th sample:
    \begin{align}
    &\bsz_k^{(k)} = (\bw \bsh_k)^{(k)}
    = \frac{K - 1}{K} M_k,
    \\
    &\bsz_k^{(m)} = (\bw \bsh_k)^{(m)} =  
    - \frac{1}{K} M_k, \quad \forall m \neq k,
    \\
    &\text{and } \sum_{m=1}^{K} \bsz_{k}^{(m)} = 0. \nonumber
    \end{align}
\end{enumerate}

\noindent If there is any $k \in [K]$ such that $\frac{N}{\sqrt{n_k}} \sqrt{\lambda_{W} \lambda_{H}} \geq \sqrt{\frac{K-1}{K}}$, then the $k$-th class-mean $\bsh_{k} = \mathbf{0}$ and all properties above still hold.
\end{theorem}

We postpone the detailed proof until Section \ref{sec:proof} in the Appendix. At a high level, our proof finds the lower bound of the loss function and studies the conditions to achieve the bound. We start by bounding the cross-entropy term to move the logit $\bsz_{k} = \bw \bsh_{k}$ out of the logarithm and exponent, using arguments based on Cauchy-Schwartz and Jensen inequalities. Next, the technical challenging part is the realization of the alignment between $\bsw_{k}$ and $K \sqrt{n_k} \bsh_{k} - \sum_{m=1}^{K} \sqrt{n_m} \bsh_{m}$ 
 to separate the weights and the features from the logits $\bsz = \bw \bsh$ in the CE loss. The constant coefficients go with each logit vector are also chosen carefully to be able to sum all logit vectors altogether, with a different amount from each class due to class-imbalance, without violating the subsequent equal conditions. After the separation of the weights and the features in logit terms, we leverage zero gradient condition of critical points to further simplify the loss function to have features as the remaining optimization variables. Then we finish the bounding and study the equal conditions.

We discuss the implications of Theorem \ref{thm:CE_main} as following.

\vspace{0.5 em}
\noindent
\textbf{Optimal features form a General Orthogonal Frame:} As we observe from Equation \eqref{eq:NC_1} in Theorem \ref{thm:CE_main}, every global solution exhibits the $\mathcal{NC}$1, i.e., within-class features collapse to their class-mean. Under the nonnegativity constraint, the optimal features form a general orthogonal frame (GOF), which consists of pairwise orthogonal vectors but with different lengths. The geometry of the optimal features for UFM class-imbalanced training problem with MSE loss is also a GOF
\cite{dang23}, but the lengths are clearly different between two losses. For $\text{UFM}_{+}$ imbalanced with supervised contrastive loss and normalized features (i.e., $\| \bsh \| = 1$), it is observed that optimal $\bh$ exhibits OF structure with equal length class-mean vectors, irrespective to the imbalanced level \cite{kini23}. 

\vspace{0.5 em}
\noindent
\textbf{Classifier converges to scaled-and-centered class-means:} The optimal classifier $\bsw_k$ of problem \eqref{eq:UFM_plus} with CE loss does not form an orthogonal structure as in the case of MSE loss class-imbalance \cite{dang23}. Our results indicate that class $k$'s classifier, $\bsw_k$, is aligned with the scaled and centered class-mean $\bsh_k$, with scaling factor $\sqrt{n_k}$, i.e., $\bsw_k \propto K \sqrt{n_k} \bsh_k - \sum_{m=1}^{K} \sqrt{n_m} \bsh_m$. We note that the proportional ratio between $\bsw_k$ and $K \sqrt{n_k} \bsh_k - \sum_{m=1}^{K} \sqrt{n_m} \bsh_m$ is identical across $k$'s. This property generalizes the original $\mathcal{NC}3$ - Convergence to self-duality property,  $\bsw_k \propto K \bsh_K - \sum_{m=1}^{K} \bsh_m$ in class-balanced setting. From Eqn. \eqref{eq:feature_ortho}, \eqref{eq:feature_norm} and \eqref{eq:NC2}, we can readily derive the $\mathcal{NC}_{2}$ - Geometry of the class-means and the classifiers in this setting. We prove that the original $\mathcal{NC}_{2}$ property in class-balanced setting is a special case of our result in Corrolary \ref{cor:balanced_setting} below.

\vspace{0.5 em}
\noindent
\textbf{Logit matrix and Margin:} Each column $\bsz_k$ of the logit matrix $\bz = \bw \bh$ is of a factor of the vector $\mathbf{y}_k - \frac{1}{K} \mathbf{1}_k$, but the factors are different among classes. Thus, the optimal matrix $\bz = \bw \bh$ of the problem \eqref{eq:UFM_plus} is different from the SELI geometry, i.e., the global structure of the UFM-SVM imbalanced problem. This observation further confirms Proposition 1 in \cite{Christos22}, which asserts that SELI is not the optimal structure for the CE imbalanced problem for any finite regularization parameter $\lambda > 0$. Furthermore, we find that the optimal classifier weight and features of the problem \eqref{eq:UFM_plus} are also different from those of SELI, for both finite ($\lambda > 0$) and vanishing regularization levels ($\lambda \rightarrow 0$). See Appendix \ref{sec:comparison_SELI} for the details.

The margin for any data point $\boldsymbol{x}_{k,i}$ from class $k$ is:
    \begin{align}
    q_{k,i}(\bw, \bh) = \bsw_k^{\top} \bsh_{k,i} - \max_{j \neq k} \bsw_j^{\top} \bsh_{k,i} = 
    M_k.
    \end{align}

We derive the results of Theorem \ref{thm:CE_main} in the special case of balanced dataset as follows. 

\begin{corollary}[Balanced dataset as a special case]
\label{cor:balanced_setting}
    Under balanced setting where $n_1 = n_2 = \ldots = n_K$, we have from Eqn. \eqref{eq:NC2} that $\bw \propto (\overline{\bh} - \bsh_{G} \mathbf{1}_{K}^{\top})^{\top}$, and thus,
    \begin{align}
        &\overline{\bh}^{\top} \overline{\bh} \propto \mathbf{I}_{K}, \nonumber \\
        %
        &\bw \bw^{\top}
        \propto
        (\overline{\bh} - \bsh_{G} \mathbf{1}_{K}^{\top})^{\top}
        (\overline{\bh} - \bsh_{G} \mathbf{1}_{K}^{\top})
        \propto \mathbf{I}_{K} -
        \frac{1}{K} \mathbf{1}_{K} \mathbf{1}_{K}^{\top}, 
        \nonumber \\
        &\bz = \bw \overline{\bh} 
         \propto  (\overline{\bh} - \bsh_{G} \mathbf{1}_{K}^{\top})^{\top} \overline{\bh}
         \propto 
         \mathbf{I}_{K} -
        \frac{1}{K} \mathbf{1}_{K} \mathbf{1}_{K}^{\top}. 
        \nonumber
    \end{align}
\end{corollary}

\begin{proof}
    The results are directly obtained from Theorem \ref{thm:CE_main} and by noting that $M_1 = M_2 = \ldots = M_K$. When $\bh$ forms an OF, the center class-mean matrix $\bh - \bsh_{G} \mathbf{1}_{K}^{\top}$ is a simplex ETF. This follows from
    \begin{align}
        &(\bh - \bsh_{G} \mathbf{1}_{K}^{\top})^{\top}
        (\bh - \bsh_{G} \mathbf{1}_{K}^{\top})
        \nonumber \\
        &= (\mathbf{I}_{K} -
        \frac{1}{K} \mathbf{1}_{K} \mathbf{1}_{K}^{\top})^{\top} \bh^{\top} \bh (\mathbf{I}_{K} -
        \frac{1}{K} \mathbf{1}_{K} \mathbf{1}_{K}^{\top}) \nonumber \\
        &\propto  (\mathbf{I}_{K} -
        \frac{1}{K} \mathbf{1}_{K} \mathbf{1}_{K}^{\top})^{\top} 
        \mathbf{I}_{K}
        (\mathbf{I}_{K} -
        \frac{1}{K} \mathbf{1}_{K} \mathbf{1}_{K}^{\top}) \nonumber \\
        &= \mathbf{I}_{K} -
        \frac{1}{K} \mathbf{1}_{K} \mathbf{1}_{K}^{\top}.
        \nonumber
    \end{align}

    \noindent The logit matrix $\bz$ also forms an ETF structure because
    \begin{align}
        \bz &= \bw \bh
        \propto
        (\bh - \bsh_{G} \mathbf{1}_{K}^{\top})^{\top} \bh \nonumber \\
        &=
        (\mathbf{I}_{K} -
        \frac{1}{K} \mathbf{1}_{K} \mathbf{1}_{K}^{\top})^{\top} \bh^{\top}
        \bh = \mathbf{I}_{K} -
        \frac{1}{K} \mathbf{1}_{K} \mathbf{1}_{K}^{\top}. \nonumber
    \end{align}
We obtain the conclusion of Corollary~\ref{cor:balanced_setting}.
\end{proof}

For the special case where dataset is balanced, Theorem \ref{thm:CE_main} recovers the ETF structure for classifier matrix $\bw$ and logit matrix $\bz$. The optimal class-mean matrix forms an orthogonal frame since it is constrained to be on non-negative orthant. 

\subsection{Classifier Norm and Angle}

\noindent By expressing the geometry of the optimal solutions explicitly, Theorem \ref{thm:CE_main} allows us to derive closed-form expressions for the norms and angles between any individual classifiers and features. Under imbalanced regime, the $k$-th class classifier $\bsw_k \propto K \sqrt{n_k} \bsh_k - \sum_{m=1}^{K} \sqrt{n_m} \bsh_m$, indicating that its norm is positively correlated with the number of sample $n_k$. We study the norm and angle of the classifier in the following proposition.

\begin{proposition}[Classifers norm and angle]
    \label{prop:classifier_norm}
    Let $\alpha = \frac{1}{K \sqrt{K (K - 1)}}
        \sqrt{\frac{\lambda_H}{\lambda_W}}$.
    The optimal classifier $\{ \bsw_k \}_{k=1}^{K}$ of problem \eqref{eq:UFM_plus} obeys:
    \begin{align}
        \| \bsw_k \|^2 &=
        \alpha \left(
        (K - 1)^2 \sqrt{n_k} M_k + \sum_{m \neq k}
        \sqrt{n_m} M_m
        \right),
        \nonumber \\
        \bsw_k^{\top} \bsw_j
        &=
        \alpha
        \Bigg[
        - (K - 1)
         \sqrt{n_k} M_k
        - (K - 1)
         \sqrt{n_j} M_j  
         \nonumber \\
         &+ \sum_{m \neq k, j}
         \sqrt{n_m} M_m \Bigg], \forall \: k \neq j, \nonumber 
        \\
        \cos(\bsw_k, \bsw_j) 
        &= \frac{\bsw_k^{\top} \bsw_j}{\| \bsw_k \| \| \bsw_j \|}.
        \nonumber 
    \end{align}
\end{proposition}

\noindent 

The fact that the weight norm is positively correlated with the number of training instances has been studied in the literature~\cite{Kang19, huang16, kim20}. Our result in Proposition \ref{prop:classifier_norm} supports this observation. We proceed to derive the norm ratio of the classifier weight and class-means and the angles of the classifiers when we will have only two group of classes with equal number of samples in each group.   

\begin{corollary}[Norm ratios]
    \label{cor:norm_ratio}
    Suppose $d \geq K$ and $(\bw, \bh)$ is a global minimizer of problem \eqref{eq:UFM_plus}. 
    Then, for any $i, j \in [K]$, we have
    \begin{align}
        \frac{\| \bsw_i \|^2}{ \| \bsw_j \|^2}
        &= 
        \frac{ (K - 1)^2 \sqrt{n_i} M_i + \sum_{m \neq i}
        \sqrt{n_m} M_m}
        {
         (K - 1)^2 \sqrt{n_j} M_j + \sum_{m \neq j}
        \sqrt{n_m} M_m
        } \nonumber \\
        \frac{\| \bsh_i \|^2}{ \| \bsh_j \|^2}
        &= \sqrt{\frac{n_j}{n_i}} \frac{M_i}{M_j}
    \end{align}
    As a consequence, if $n_i \geq n_j$, $ \| \bsw_i \| \geq \| \bsw_j \|$.
\end{corollary}

\begin{proof}
    The results are direct consequences of Proposition \ref{prop:classifier_norm}.
\end{proof}
    
For the norm of optimal features, one might expect it is negatively correlated with the number of training samples and the results for UFM-MSE and UFM-SVM training problem agree with this expectation \cite{dang23, Christos22}. However, for CE loss, we find that this statement is not always true because the function $M_i / \sqrt{n_i}$ is not always a decreasing function with respect to $n_i$.

\begin{corollary}[Classifer angles]
    \label{cor:angle}
     Assume the dataset has $K_A$ majority classes with $n_A$ samples per class and $K_B$ minority classes with $n_B$ samples per class, then we have
     \begin{align}
        \label{eq:angle1}
         &\cos(\bsw_{\text{major}}, \bsw_{\text{major}}^{\prime})  \\
         &= 1 - \frac{K^2 \sqrt{n_A} M_A}
         {K(K-1) \sqrt{n_A} M_A
        - K_B \sqrt{n_A} M_A + K_B \sqrt{n_B} 
         M_B}, \nonumber \\
          \label{eq:angle2}
         &\cos(\bsw_{\text{minor}}, \bsw_{\text{minor}}^{\prime}) \\
         &= 1 - \frac{K^2 \sqrt{n_B} M_B}{K(K-1) \sqrt{n_B} M_B
         - K_A \sqrt{n_B} M_B 
         + K_A \sqrt{n_A} M_A}. \nonumber
     \end{align}
     Consequently, we have:
     \begin{align}
          \cos(\bsw_{\text{major}}, \bsw_{\text{major}}^{\prime}) < \cos(\bsw_{\text{minor}}, \bsw_{\text{minor}}^{\prime}).  \nonumber
     \end{align}
\end{corollary}

\begin{proof}
    The results are direct consequences of Proposition \ref{prop:classifier_norm}.
\end{proof}

From Corollary \ref{cor:angle}, we deduce that the angle between classifier of major classes will form larger angles than those of minor classes. 
This observation further explains the smaller volume of feature space allocated for minority classes, which is one of the main reasons for the drop in model performance for these classes. Additionally, since Eqn. \eqref{eq:angle1} and \eqref{eq:angle2} are true for any pair of classes within the same category (i.e., major or minor), this means that classifiers of classes with the same number of training instances have the \textit{same pairwise angle}.   

\subsection{Heavy Imbalances Cause Minority Collapse and Complete Collapse}

Data naturally exhibit imbalance in their class distribution. Models trained on highly-skewed class distribution data tends to be biased towards the majority classes, resulting in poor performance on the minority classes \cite{huang16, Kang19, kim20}. Especially, \cite{Fang21} observes that when the imbalance ratio $R := n_{\text{major}} / n_{\text{minor}}$ is larger than some threshold, the angle between minority classifiers becomes zero, and these classifiers have the same length. Consequently, these classifiers become indistinguishable and the network would predict the same probabilities for these minor classes. This phenomenon is termed as Minority Collapse. From Theorem \ref{thm:CE_main}, we obtain the exact threshold of the Minority Collapse occurrence for every class for training problem \eqref{eq:UFM_plus}, in terms of the number of training samples and hyperparameters.

\begin{corollary}[Minority Collapse and Complete Collapse]
    \label{cor:minority_collapse}
    For any class $k \in [K]$, if $n_k \leq C(N, K, \lambda_W, \lambda_H) := N^2 \frac{K}{K - 1} \lambda_W \lambda_H$, then at the optimal solution of problem \eqref{eq:UFM_plus}, $\bsh_k = 0$ and $\bsw_{k} = \bsw_{k^{\prime}}$ for any $k^{\prime} \in [K]$ such that $n_{k^{\prime}} \leq C(N, K, \lambda_W, \lambda_H)$. 
    
    \begin{enumerate}[label=(\alph*)]
    \item Minority Collapse: 
    If the dataset has $K_A$ majority classes with $n_A$ samples per class and $K_B$ minority classes with $n_B$ samples per class, then Minority Collapse happens if the imbalance ratio
        \begin{align}
            R := \frac{n_A}{n_B} \geq 
            \frac{1}{K_A}
            \left(\frac{K-1}{NK \lambda_W \lambda_H} - K_B
            \right), 
        \end{align}
        with $K = K_A + K_B$ and $N = n_A K_A + n_B K_B$.

    \item Complete Collapse: 
    If
    \begin{align}
    \label{eq:head_class_collapse}
    \frac{N^2}{n_A} \geq \frac{K-1}{K \lambda_W \lambda_H},
    \end{align}
    then all classes collapse and the optimal solution is trivial, i.e., $(\bw, \bh) = (\mathbf{0}, \mathbf{0})$.  
    \end{enumerate}
\end{corollary}

\begin{proof}
    The results are direct consequences of Theorem \ref{thm:CE_main}.
\end{proof}

Corollary \ref{cor:minority_collapse} implies that even the head classes will collapse when the ratio $N^2 / n_A$ is large enough, i.e., the dataset has a huge amount of samples or has too many classes. The bound \eqref{eq:head_class_collapse} suggests that we should lower the regularization level to avoid this \textit{complete collapse} phenomenon.

\section{Experimental Results}
\label{sec:experimental_results}
\subsection{Metric definition}
\label{sec:metric}
We recall the notation $\mathbf{h}_{k} \coloneqq \frac{1}{n} \sum_{i=1}^n \mathbf{h}_{k,i}$, i.e., the class-means of class $k$ and $\mathbf{h}_{G} \coloneqq \frac{1}{K n} \sum_{k=1}^K \sum_{i=1}^n \mathbf{h}_{k,i}$ is the feature global-mean. We calculate the within-class covariance matrix $\mathbf{\Sigma}_W \coloneqq \frac{1}{N} \sum_{k=1}^{K} \sum_{i=1}^{n} (\mathbf{h}_{k,i} - \mathbf{h}_{k}) (\mathbf{h}_{k,i} - \mathbf{h}_{k})^{\top}$ and the between-class covariance matrix $\mathbf{\Sigma}_B \coloneqq \frac{1}{K} \sum_{k=1}^K (\mathbf{h}_{k} - \mathbf{h}_{G}) (\mathbf{h}_{k} - \mathbf{h}_{G})^{\top}$.

\noindent \textbf{Feature collapse.} Following previous works \cite{papyan20, han21, Zhu21, Tirer22}, we measure feature collapse using $\mathcal{NC}1$ metric
\begin{align}
\mathcal{NC}1 \coloneqq \frac{1}{K} \text{trace} (\mathbf{\Sigma}_W \mathbf{\Sigma}_B^\dagger), \nonumber
\end{align}
where $\mathbf{\Sigma}_B^\dagger$ is the Moore-Penrose inverse of $\mathbf{\Sigma}_B$.

\noindent \textbf{Relation between the classifier $\bw$ and features $\bh$.} 
To verify the relation in Eqn. \eqref{eq:NC2} in Theorem \ref{thm:CE_main}, we measure the similarity between the learned classifier $\bW$ and the $\text{UFM}_{+}$ structure described as follows
\begin{align}
            \label{eq:metric_W}
            &\mathcal{NC}2 - 
            (\mathbf{W} - \overline{\mathbf{H}}^{\top})
            \coloneqq \left\| \frac{\bw}{\| \bw \|_F} 
            - \frac{\bw_{\text{UFM}_{+}} (\bh) }{\| \bw_{\text{UFM}_{+}} (\bh) \|_F} \right\|_F, \nonumber \\
            &\text{where } 
            \bw_{\text{UFM}_{+}} (\bh) = 
            \begin{bmatrix}
            K \sqrt{n_1} \bsh_1^{\top} - \sum_{m=1}^{K} \sqrt{n_m}\bsh_m^{\top}    
            \\
            K \sqrt{n_2} \bsh_2^{\top} - \sum_{m=1}^{K} \sqrt{n_m}\bsh_m^{\top} 
            \\
            \ldots
            \\
            K \sqrt{n_K} \bsh_K^{\top} - \sum_{m=1}^{K} \sqrt{n_m}\bsh_m^{\top} \nonumber
            \end{bmatrix}.
\end{align}

\textbf{Classifier and Class-means Gram matrix.} We verify the geometry of the classifier and class-means matrix as follows,
\begin{align}
 &\mathcal{NC}2 - (\bw \bwt) \coloneqq
    \left\| 
    \frac{\bw \bwt}{\| \bw \bwt \|_F} - \frac{\bw \bwt_{\text{UFM}}}{\| \bw \bwt_{\text{UFM}} \|_F}
    \right\|,
    \nonumber \\
&\mathcal{NC}2 - (\overline{\bh}^{\top} \overline{\bh}) \coloneqq 
\left\| 
\frac{\overline{\bh}^{\top} \overline{\bh}}{\| \overline{\bh}^{\top} \overline{\bh} \|_F} 
- \frac{\overline{\bh}^{\top} \overline{\bh}_{\text{UFM}}}{\| \overline{\bh}^{\top} \overline{\bh}_{\text{UFM}} \|_F}
\right\|, \nonumber
\end{align}
where $\bw \bwt_{\text{UFM}}$ and $\overline{\bh}^{\top} \overline{\bh}_{\text{UFM}}$ are derived from Theorem \ref{thm:CE_main} (see details in Appendix \ref{sec:additional_exp}).

\noindent \textbf{Prediction matrix $\bW \overline{\bh}$.} To measure the similarity of the learned $\bz = \bW \overline{\bh}$ to the $\text{UFM}_{+}$ structure described in Theorem~\ref{thm:CE_main}, we define $\mathcal{NC}3$ metric as follows
\begin{align}
            &\mathcal{NC}3 - (\bw \overline{\bh}) \coloneqq \left\| \frac{\bw \overline{\bh}}{\| \bw \overline{\bh} \|_{F}} -  \frac{\bw \overline{\bh}_{\text{UFM}_{+}}}{\| \bw \overline{\bh}_{\text{UFM}_{+}} \|_{F}} \right\|_{F}, \nonumber 
\end{align}
where $\bw \overline{\bh}_{\text{UFM}_{+}}$ are described in Appendix \ref{sec:additional_exp}.

\subsection{Experiment details}
\label{subsec:experiment_details}
To verify our theoretical results, we train networks that mimic the UFM setting with ReLU features as in Eqn.\eqref{eq:UFM_plus}. In particular, we use a 6-layer multilayer perceptron (MLP) model with ReLU activation, VGG11~\cite{Simonyan14}, ResNet18~\cite{DBLP:conf/cvpr/HeZRS16} as our three main backbone feature extractors. We train these models on the imbalanced subsets of 4 datasets: MNIST, FashionMNIST, CIFAR10, and CIFAR100. Then we measure the evolution of five $\mathcal{NC}$ metrics in Section \ref{sec:metric} to study the geometry of the last-layer features and the classifier. Due to space consideration, we show the results for CIFAR10 and CIFAR100 below. The remaining experiments and the training details can be found in Appendix~\ref{sec:additional_exp}.


\begin{figure}
\centering
\begin{subfigure}{0.42\textwidth}
    \includegraphics[width=\textwidth]{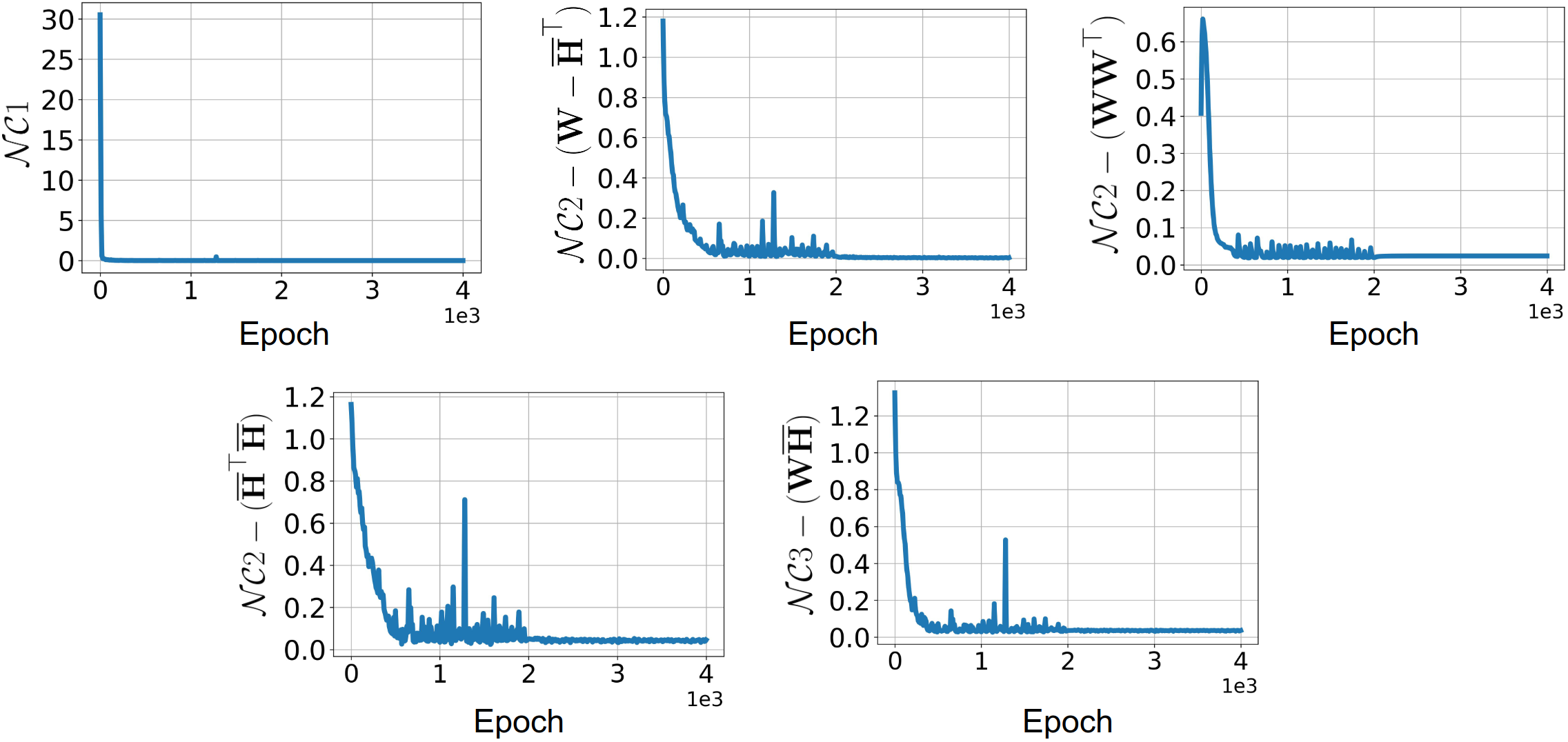}
    \caption{\small MLP}
    \label{fig:mlp_cifar10}
\end{subfigure}
\begin{subfigure}{0.42\textwidth}
    \includegraphics[width=\textwidth]{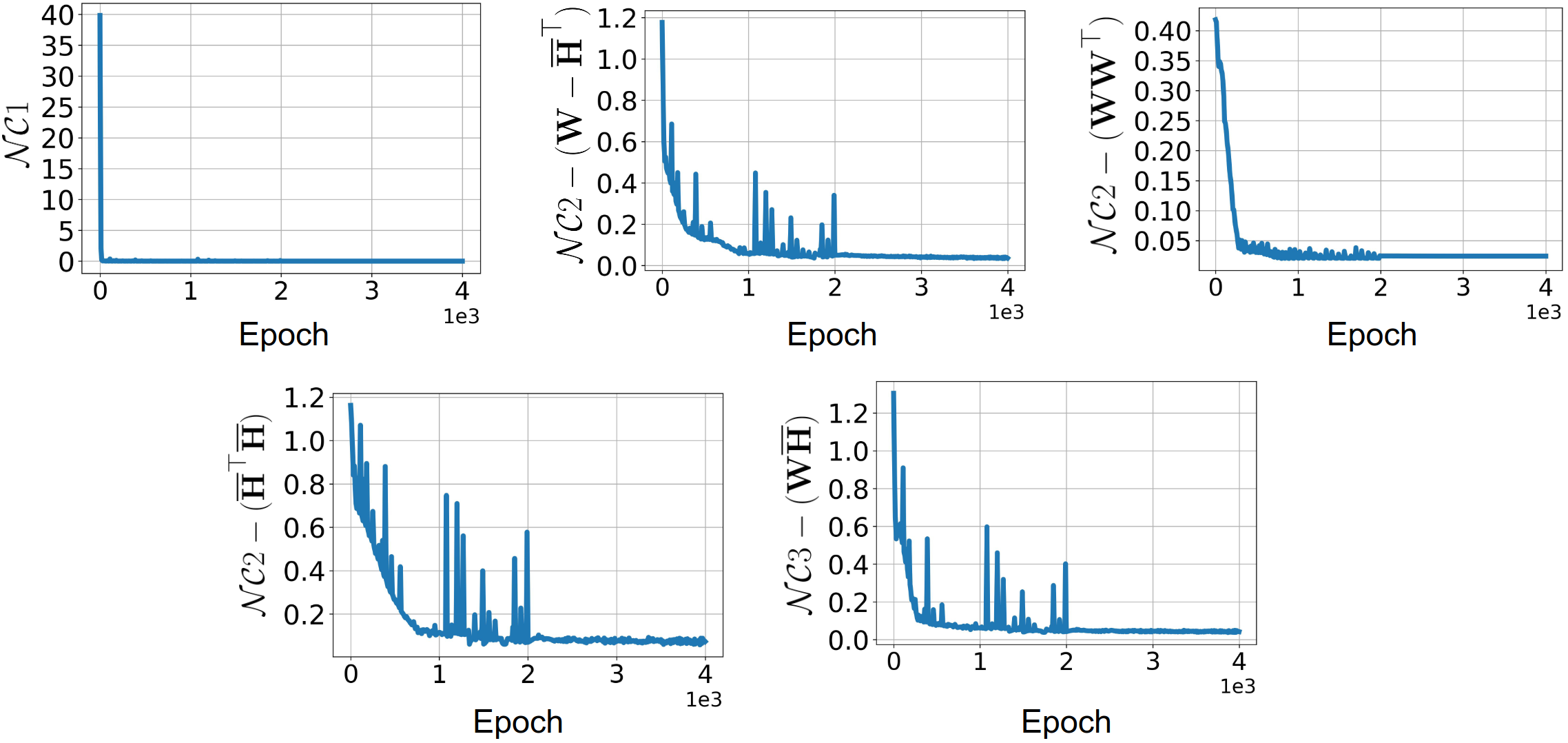}
    \caption{\small VGG11}
    \label{fig:vgg11_cifar10}
\end{subfigure}
\begin{subfigure}{0.42\textwidth}
    \includegraphics[width=\textwidth]{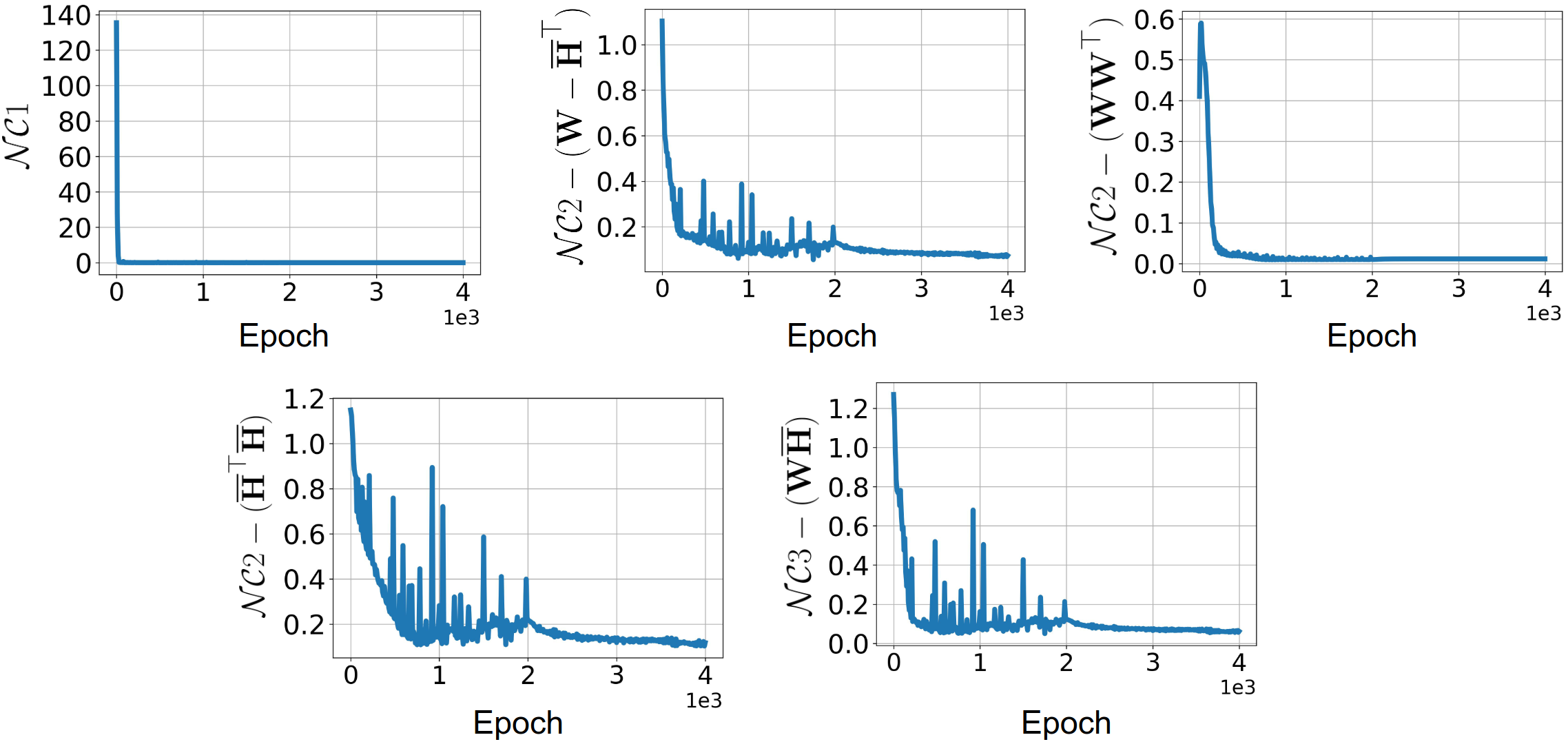}
    \caption{\small ResNet18}
    \label{fig:resnet18_cifar10}
\end{subfigure}
\vspace{-0.1in}
\caption{\small 
$\mathcal{NC}$ metrics evolution for three models trained on imbalanced subset of CIFAR10 dataset with cross entropy loss.}
\label{fig:cifar10_3_models}
\vspace{-0.2in}
\end{figure}

\begin{figure}[t]
\centering
\begin{subfigure}{0.42\textwidth}
    \includegraphics[width=\textwidth]{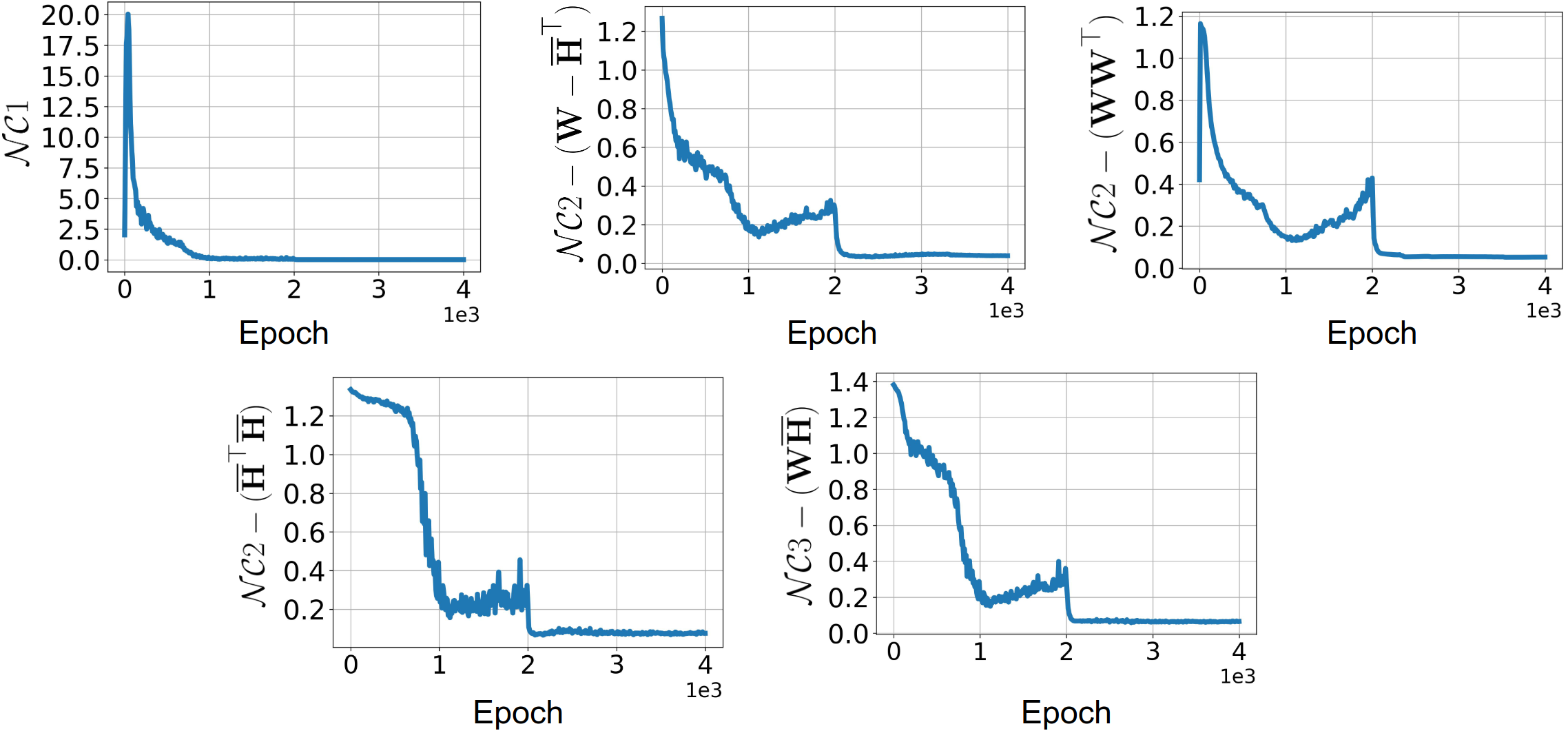}
    \caption{\small MLP}
    \label{fig:mlp_cifar100}
\end{subfigure}
\hspace{0.1in}
\begin{subfigure}{0.42\textwidth}
    \includegraphics[width=\textwidth]{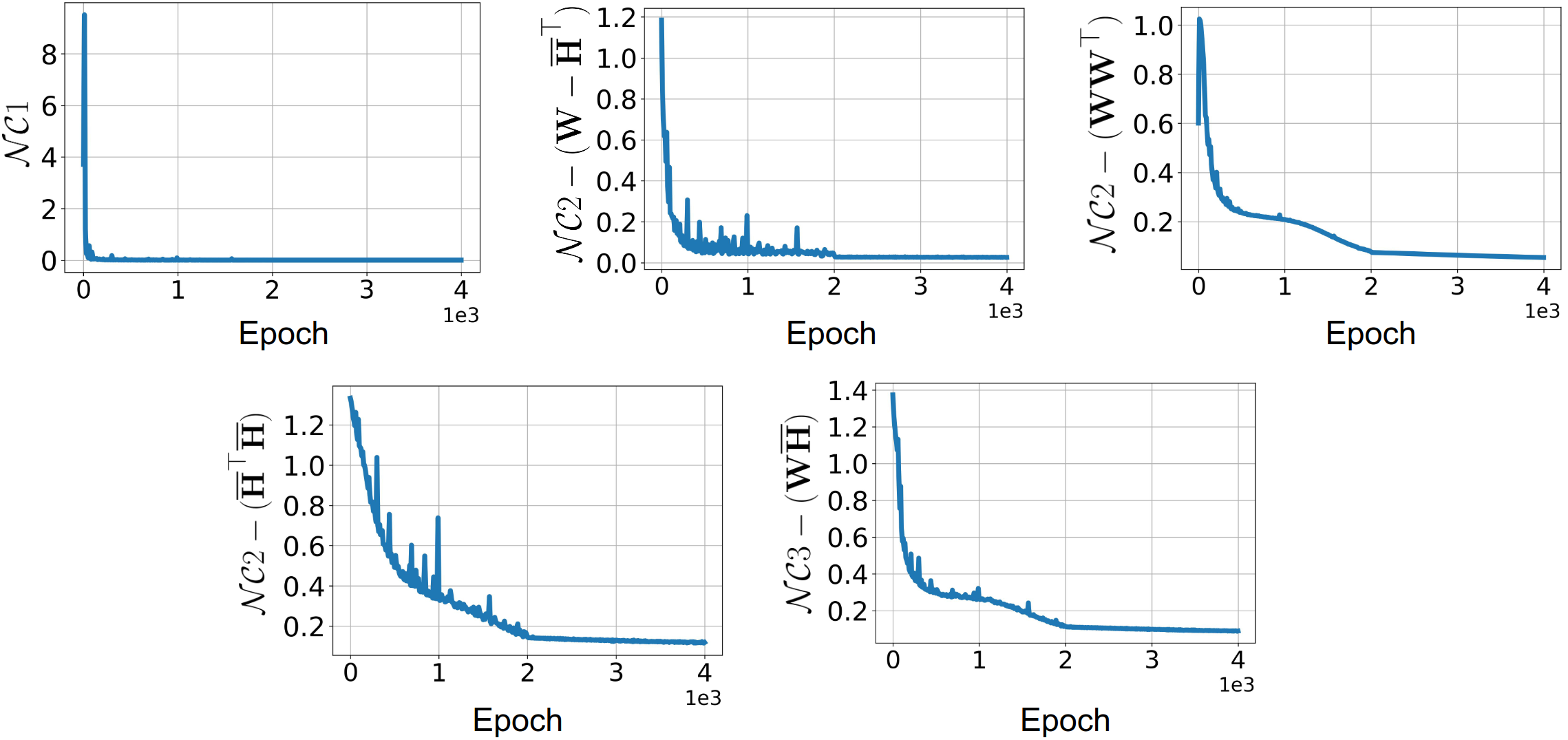}
    \caption{\small VGG11}
    \label{fig:vgg11_cifar100}
\end{subfigure}
\begin{subfigure}{0.42\textwidth}
    \includegraphics[width=\textwidth]{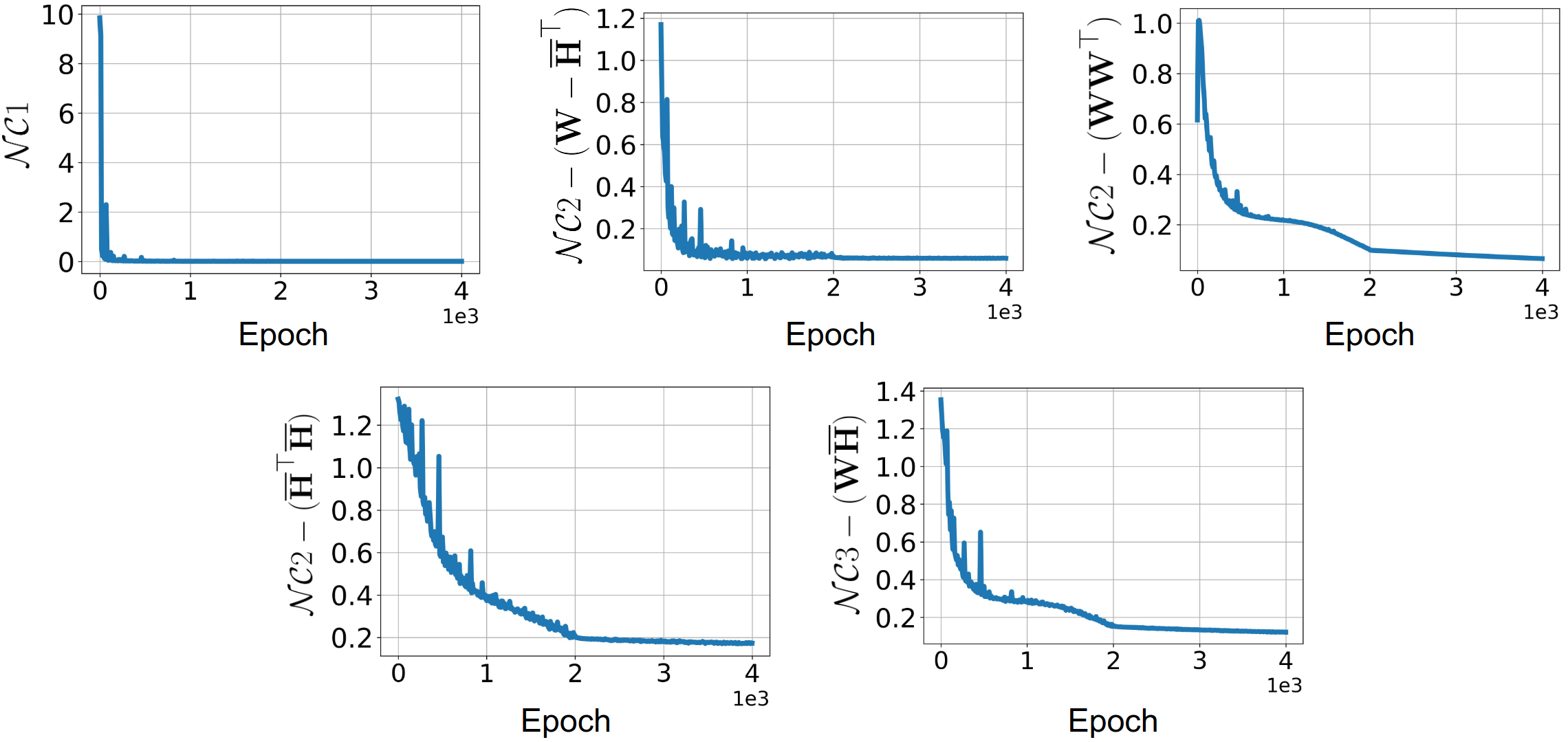}
    \caption{\small ResNet18}
    \label{fig:resnet18_cifar100}
\end{subfigure}
\vspace{-0.1in}
\caption{\small $\mathcal{NC}$ metrics evolution for three models trained on imbalanced subset of CIFAR100 dataset with cross entropy loss.}
\label{fig:cifar100_3_models}
\vspace{-0.2in}
\end{figure}


\textbf{Image classification experiment on CIFAR10:} For this experiment, a subset of the CIFAR10 dataset with $\{1000, 1000, 2000, 2000, 3000, 3000,$ $4000, 4000, 5000, 5000\}$ random samples per class is utilized as training data. We train each backbone model with Adam optimizer with batch size $256$, the weight decay is $\lambda_{W} = \num{1e-4}$. Feature decay $\lambda_{H}$ is set to $\num{1e-5}$ for MLP and VGG11, and to $\num{1e-4}$ for ResNet18. 
In Figure~\ref{fig:cifar10_3_models}, we observe the convergence of $\mathcal{NC}$ metrics to small values as training progresses, which corroborates our theoretical prediction.

\textbf{Image classification experiment on CIFAR100:} We create a random subset of the CIFAR100 dataset with 100 samples per class for the first 20 classes, 200 samples per class for the next 20 classes,..., 500 samples per class for the remaining 20 classes.
Each backbone model is then trained with Adam optimizer with batch size $256$, the learning rate is  $\num{2e-4}$ for VGG11, ResNet18 and $\num{1e-4}$ for MLP. Weight decay $\lambda_{W}$ and feature decay $\lambda_{H}$ is set to $\num{1e-4}$ and $\num{1e-5}$, respectively.
Figure~\ref{fig:cifar100_3_models} empirically verifies Theorem~\ref{thm:CE_main} in this setting with a large number of classes ($K = 100$).

\textbf{Varying imbalance ratio:} In this experiment, we validate our theoretical predictions for multiple levels of data imbalance. We train MLP and VGG models on random subsets of the CIFAR10 and MNIST datasets with varying imbalance ratios ($R=5, 10, 20, 50$). Figures~\ref{fig:r_mnist} and \ref{fig:r_cifar10} demonstrate the convergence of $\mathcal{NC}$ metrics for MNIST and CIFAR10 datasets, respectively.

\begin{figure}
\centering
\begin{subfigure}{0.42\textwidth}
    \includegraphics[width=\textwidth]{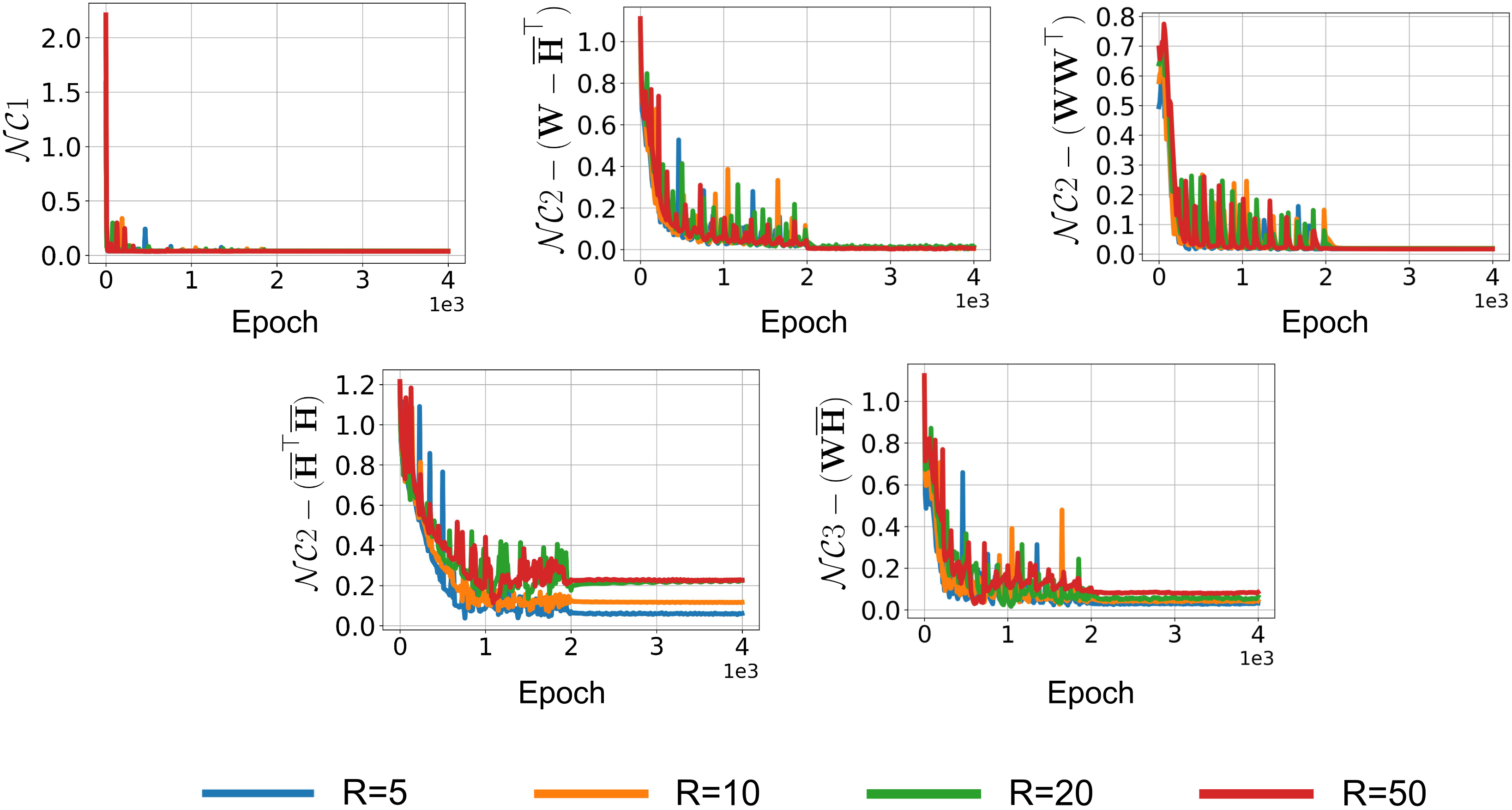}
    \vspace{-0.2in}
    \caption{\small MLP}
    \label{fig:r_mlp_mnist}
\end{subfigure}
\begin{subfigure}{0.42\textwidth}
    \includegraphics[width=\textwidth]{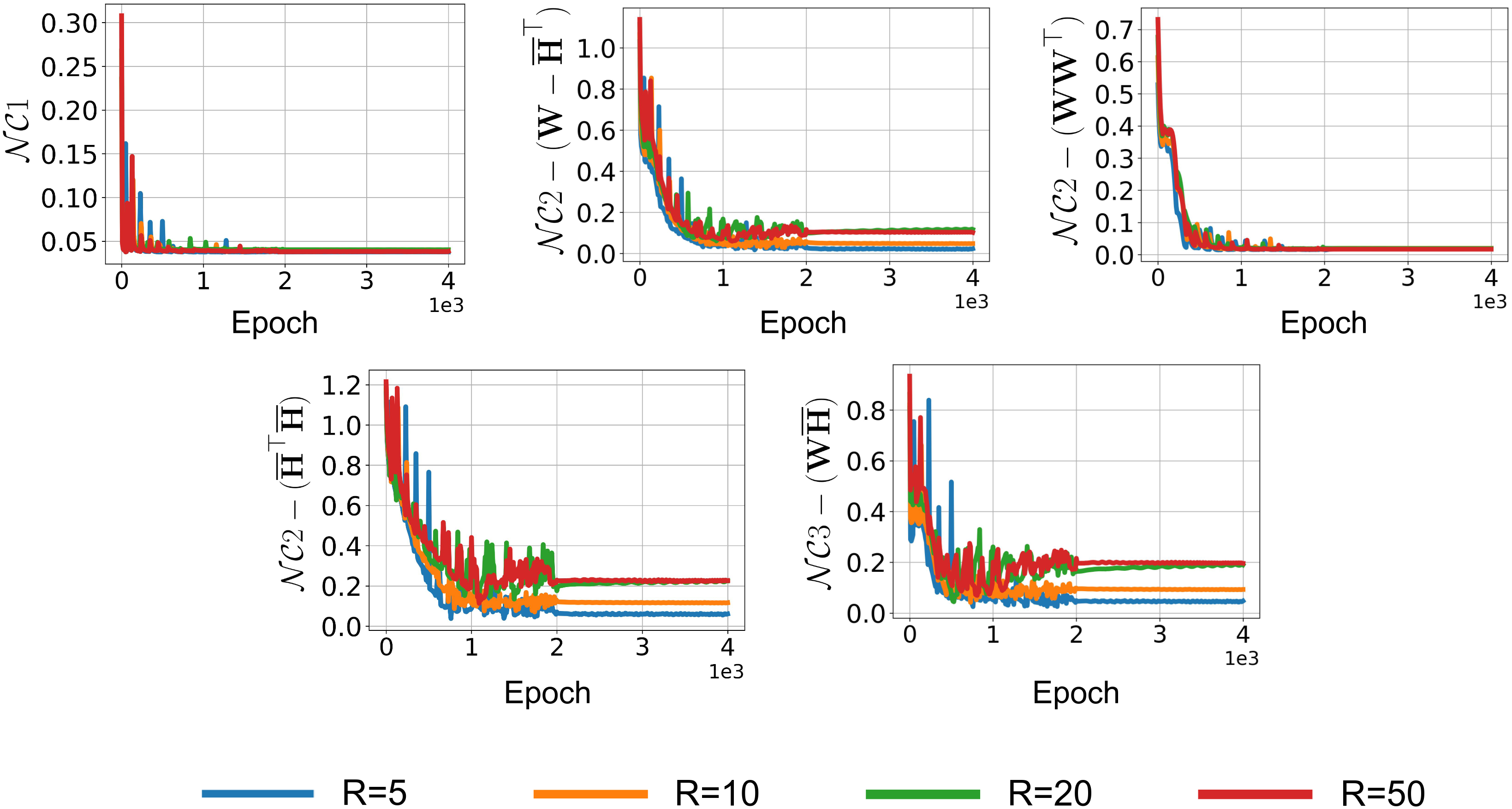}
    \vspace{-0.2in}
    \caption{\small VGG11}
    \label{fig:r_vgg11_mnist}
\end{subfigure}
\vspace{-0.1in}
\caption{\small $\mathcal{NC}$ metrics evolution of MLP and VGG11 backbone trained on MNIST imbalanced subsets with cross entropy loss}
\label{fig:r_mnist}
\vspace{-0.2in}
\end{figure}

\noindent \textbf{Illustration of $\overline{\bh}^{\top} \overline{\bh}$:} We normalize the $\overline{\bh}^{\top} \overline{\bh}$ matrix obtained from the last epoch of the MLP model trained on the CIFAR10 dataset. The orthogonal structure of the learned features along with the theoretical prediction derived from Theorem~\ref{thm:CE_main} are demonstrated in Figure~\ref{fig:MLP_CE_nobias_HTH}. 


\begin{figure}[h]
\centering
\includegraphics[width=0.45\textwidth]{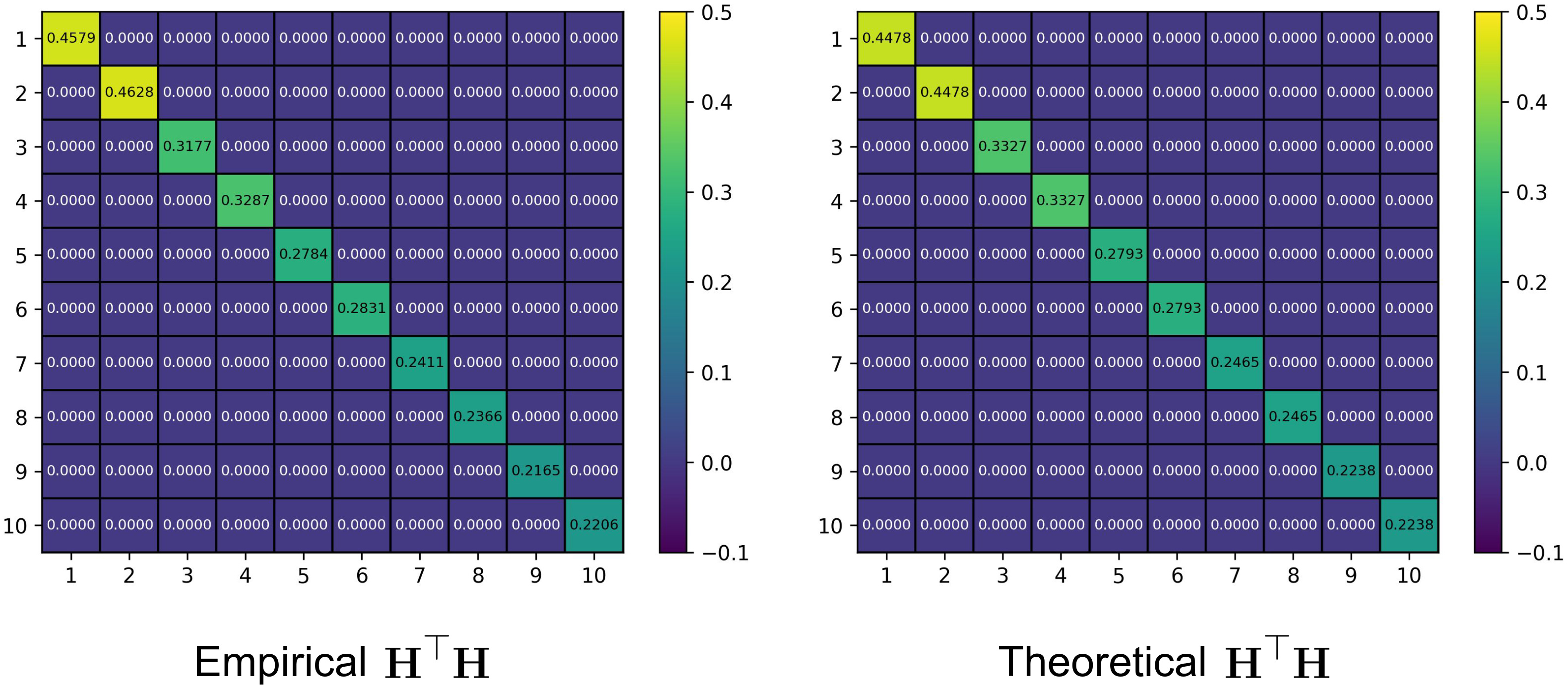}
\caption{$\overline{\bh}^{\top} \overline{\bh}$ matrix extracted from the last epoch of the trained MLP model.}
\label{fig:MLP_CE_nobias_HTH}
\vspace{-0.2in}
\end{figure}

\section{Concluding Remarks}
\label{sec:conclusion}

In this work, we present a rigorous and explicit study of Neural Collapse phenomenon in the setting of imbalanced dataset using unconstrained non-negative features model and cross-entropy loss. In particular, we provide a closed-form characterization of the last-layer features and classifier weights learned by the network training. We find that while the variability collapse property still holds, the geometry of the learned features and learned classifier weights are different from the original definition of Neural Collapse, due to the class-imbalance of the training data. Specifically, we prove that at optimality, the features form an orthogonal structure while the classifier weights are aligned to the scaled and centered class-means, which generalizes the original definition of Neural Collapse in class-balanced settings. Furthermore, with closed-form derivations of the solution, we are able to quantify the norms and the angles between the learned features and classifier weights across class distribution. As a limitation, we only study the convergence geometries under the condition that the feature dimension $d$ is at least the number of classes $K$. The geometric structure of the features and classifier in the bottleneck situation $d < K$ is still unaddressed and we leave it for future work.

\section*{Acknowledgements}
This research/project is supported by the National Research Foundation Singapore under the AI
Singapore Programme (AISG Award No: AISG2-TC-2023-012-SGIL). NH acknowledges support from the NSF IFML 2019844 and the NSF AI Institute for Foundations of Machine Learning.

\section*{Impact Statement}
This paper presents work whose goal is to advance the field of Machine Learning. There are many potential societal consequences of our work, none which we feel must be specifically highlighted here.




\bibliography{icml2024}

\begin{thebibliography}{34}
\providecommand{\natexlab}[1]{#1}
\providecommand{\url}[1]{\texttt{#1}}
\expandafter\ifx\csname urlstyle\endcsname\relax
  \providecommand{\doi}[1]{doi: #1}\else
  \providecommand{\doi}{doi: \begingroup \urlstyle{rm}\Url}\fi

\bibitem[Behnia et~al.(2023)Behnia, Kini, Vakilian, and
  Thrampoulidis]{Behnia23}
Behnia, T., Kini, G.~R., Vakilian, V., and Thrampoulidis, C.
\newblock On the implicit geometry of cross-entropy parameterizations for
  label-imbalanced data.
\newblock In \emph{International Conference on Artificial Intelligence and
  Statistics}, pp.\  10815--10838. PMLR, 2023.

\bibitem[Cao et~al.(2019)Cao, Wei, Gaidon, Arechiga, and Ma]{Cao19}
Cao, K., Wei, C., Gaidon, A., Arechiga, N., and Ma, T.
\newblock Learning imbalanced datasets with label-distribution-aware margin
  loss, 2019.
\newblock URL \url{https://arxiv.org/abs/1906.07413}.

\bibitem[Dang et~al.(2023)Dang, Nguyen, Tran, Tran, and Ho]{dang23}
Dang, H., Nguyen, T., Tran, T., Tran, H., and Ho, N.
\newblock Neural collapse in deep linear network: From balanced to imbalanced
  data.
\newblock \emph{arXiv preprint arXiv:2301.00437}, 2023.

\bibitem[Ergen \& Pilanci(2020)Ergen and Pilanci]{Ergen20}
Ergen, T. and Pilanci, M.
\newblock Revealing the structure of deep neural networks via convex duality,
  2020.
\newblock URL \url{https://arxiv.org/abs/2002.09773}.

\bibitem[Fang et~al.(2021)Fang, He, Long, and Su]{Fang21}
Fang, C., He, H., Long, Q., and Su, W.~J.
\newblock Exploring deep neural networks via layer-peeled model: Minority
  collapse in imbalanced training.
\newblock \emph{Proceedings of the National Academy of Sciences}, 118\penalty0
  (43), oct 2021.
\newblock \doi{10.1073/pnas.2103091118}.
\newblock URL \url{https://doi.org/10.1073%2Fpnas.2103091118}.

\bibitem[Graf et~al.(2023)Graf, Hofer, Niethammer, and Kwitt]{graf23}
Graf, F., Hofer, C.~D., Niethammer, M., and Kwitt, R.
\newblock Dissecting supervised contrastive learning, 2023.

\bibitem[Han et~al.(2021)Han, Papyan, and Donoho]{han21}
Han, X.~Y., Papyan, V., and Donoho, D.~L.
\newblock Neural collapse under mse loss: Proximity to and dynamics on the
  central path, 2021.
\newblock URL \url{https://arxiv.org/abs/2106.02073}.

\bibitem[He et~al.(2016)He, Zhang, Ren, and Sun]{DBLP:conf/cvpr/HeZRS16}
He, K., Zhang, X., Ren, S., and Sun, J.
\newblock Deep residual learning for image recognition.
\newblock In \emph{2016 {IEEE} Conference on Computer Vision and Pattern
  Recognition, {CVPR} 2016, Las Vegas, NV, USA, June 27-30, 2016}, pp.\
  770--778. {IEEE} Computer Society, 2016.
\newblock \doi{10.1109/CVPR.2016.90}.
\newblock URL \url{https://doi.org/10.1109/CVPR.2016.90}.

\bibitem[Hong \& Ling(2023)Hong and Ling]{hong23}
Hong, W. and Ling, S.
\newblock Neural collapse for unconstrained feature model under cross-entropy
  loss with imbalanced data.
\newblock \emph{arXiv preprint arXiv:2309.09725}, 2023.

\bibitem[Hornik(1991)]{Hornik91}
Hornik, K.
\newblock Approximation capabilities of multilayer feedforward networks.
\newblock \emph{Neural Networks}, 4\penalty0 (2):\penalty0 251--257, 1991.
\newblock ISSN 0893-6080.
\newblock \doi{https://doi.org/10.1016/0893-6080(91)90009-T}.
\newblock URL
  \url{https://www.sciencedirect.com/science/article/pii/089360809190009T}.

\bibitem[Hornik et~al.(1989)Hornik, Stinchcombe, and White]{Hornik89}
Hornik, K., Stinchcombe, M., and White, H.
\newblock Multilayer feedforward networks are universal approximators.
\newblock \emph{Neural Networks}, 2\penalty0 (5):\penalty0 359--366, 1989.
\newblock ISSN 0893-6080.
\newblock \doi{https://doi.org/10.1016/0893-6080(89)90020-8}.
\newblock URL
  \url{https://www.sciencedirect.com/science/article/pii/0893608089900208}.

\bibitem[Huang et~al.(2016)Huang, Li, Loy, and Tang]{huang16}
Huang, C., Li, Y., Loy, C.~C., and Tang, X.
\newblock Learning deep representation for imbalanced classification.
\newblock In \emph{Proceedings of the IEEE conference on computer vision and
  pattern recognition}, pp.\  5375--5384, 2016.

\bibitem[Ji et~al.(2021)Ji, Lu, Zhang, Deng, and Su]{Ji21}
Ji, W., Lu, Y., Zhang, Y., Deng, Z., and Su, W.~J.
\newblock An unconstrained layer-peeled perspective on neural collapse, 2021.
\newblock URL \url{https://arxiv.org/abs/2110.02796}.

\bibitem[Kang et~al.(2019)Kang, Xie, Rohrbach, Yan, Gordo, Feng, and
  Kalantidis]{Kang19}
Kang, B., Xie, S., Rohrbach, M., Yan, Z., Gordo, A., Feng, J., and Kalantidis,
  Y.
\newblock Decoupling representation and classifier for long-tailed recognition,
  2019.
\newblock URL \url{https://arxiv.org/abs/1910.09217}.

\bibitem[Kang et~al.(2020)Kang, Li, Xie, Yuan, and Feng]{kang20}
Kang, B., Li, Y., Xie, S., Yuan, Z., and Feng, J.
\newblock Exploring balanced feature spaces for representation learning.
\newblock In \emph{International Conference on Learning Representations}, 2020.

\bibitem[Kim \& Kim(2020)Kim and Kim]{kim20}
Kim, B. and Kim, J.
\newblock Adjusting decision boundary for class imbalanced learning.
\newblock \emph{IEEE Access}, 8:\penalty0 81674--81685, 2020.

\bibitem[Kini et~al.(2023)Kini, Vakilian, Behnia, Gill, and
  Thrampoulidis]{kini23}
Kini, G.~R., Vakilian, V., Behnia, T., Gill, J., and Thrampoulidis, C.
\newblock Supervised-contrastive loss learns orthogonal frames and batching
  matters.
\newblock \emph{arXiv preprint arXiv:2306.07960}, 2023.

\bibitem[Liu et~al.(2023)Liu, Zhang, Hu, Cao, Yao, and Pan]{liu23}
Liu, X., Zhang, J., Hu, T., Cao, H., Yao, Y., and Pan, L.
\newblock Inducing neural collapse in deep long-tailed learning.
\newblock In \emph{International Conference on Artificial Intelligence and
  Statistics}, pp.\  11534--11544. PMLR, 2023.

\bibitem[Lu \& Steinerberger(2020)Lu and Steinerberger]{Lu20}
Lu, J. and Steinerberger, S.
\newblock Neural collapse with cross-entropy loss, 2020.
\newblock URL \url{https://arxiv.org/abs/2012.08465}.

\bibitem[Nguyen et~al.(2022)Nguyen, Levie, Lienen, H{\"u}llermeier, and
  Kutyniok]{nguyen22}
Nguyen, D.~A., Levie, R., Lienen, J., H{\"u}llermeier, E., and Kutyniok, G.
\newblock Memorization-dilation: Modeling neural collapse under noise.
\newblock In \emph{The Eleventh International Conference on Learning
  Representations}, 2022.

\bibitem[Papyan et~al.(2020)Papyan, Han, and Donoho]{papyan20}
Papyan, V., Han, X.~Y., and Donoho, D.~L.
\newblock Prevalence of neural collapse during the terminal phase of deep
  learning training.
\newblock \emph{CoRR}, abs/2008.08186, 2020.
\newblock URL \url{https://arxiv.org/abs/2008.08186}.

\bibitem[Rangamani \& Banburski-Fahey(2022)Rangamani and
  Banburski-Fahey]{Rangamani22}
Rangamani, A. and Banburski-Fahey, A.
\newblock Neural collapse in deep homogeneous classifiers and the role of
  weight decay.
\newblock In \emph{ICASSP 2022 - 2022 IEEE International Conference on
  Acoustics, Speech and Signal Processing (ICASSP)}, pp.\  4243--4247, 2022.
\newblock \doi{10.1109/ICASSP43922.2022.9746778}.

\bibitem[Simonyan \& Zisserman(2014)Simonyan and Zisserman]{Simonyan14}
Simonyan, K. and Zisserman, A.
\newblock Very deep convolutional networks for large-scale image recognition,
  2014.
\newblock URL \url{https://arxiv.org/abs/1409.1556}.

\bibitem[S{\'u}ken{\'\i}k et~al.(2023)S{\'u}ken{\'\i}k, Mondelli, and
  Lampert]{sukenik23}
S{\'u}ken{\'\i}k, P., Mondelli, M., and Lampert, C.
\newblock Deep neural collapse is provably optimal for the deep unconstrained
  features model.
\newblock \emph{arXiv preprint arXiv:2305.13165}, 2023.

\bibitem[Thrampoulidis et~al.(2022)Thrampoulidis, Kini, Vakilian, and
  Behnia]{Christos22}
Thrampoulidis, C., Kini, G.~R., Vakilian, V., and Behnia, T.
\newblock Imbalance trouble: Revisiting neural-collapse geometry, 2022.
\newblock URL \url{https://arxiv.org/abs/2208.05512}.

\bibitem[Tirer \& Bruna(2022)Tirer and Bruna]{Tirer22}
Tirer, T. and Bruna, J.
\newblock Extended unconstrained features model for exploring deep neural
  collapse, 2022.
\newblock URL \url{https://arxiv.org/abs/2202.08087}.

\bibitem[Tirer et~al.(2023)Tirer, Huang, and Niles-Weed]{Tirer23}
Tirer, T., Huang, H., and Niles-Weed, J.
\newblock Perturbation analysis of neural collapse.
\newblock In \emph{International Conference on Machine Learning}, pp.\
  34301--34329. PMLR, 2023.

\bibitem[Yang et~al.(2022)Yang, Chen, Li, Xie, Lin, and Tao]{Yang22}
Yang, Y., Chen, S., Li, X., Xie, L., Lin, Z., and Tao, D.
\newblock Inducing neural collapse in imbalanced learning: Do we really need a
  learnable classifier at the end of deep neural network?, 2022.
\newblock URL \url{https://arxiv.org/abs/2203.09081}.

\bibitem[Yarotsky(2018)]{Yarotsky18}
Yarotsky, D.
\newblock Universal approximations of invariant maps by neural networks, 2018.
\newblock URL \url{https://arxiv.org/abs/1804.10306}.

\bibitem[Ye et~al.(2020)Ye, Chen, Zhan, and Chao]{ye20}
Ye, H.-J., Chen, H.-Y., Zhan, D.-C., and Chao, W.-L.
\newblock Identifying and compensating for feature deviation in imbalanced deep
  learning.
\newblock \emph{arXiv preprint arXiv:2001.01385}, 2020.

\bibitem[Zhou(2018)]{Zhou18}
Zhou, D.-X.
\newblock Universality of deep convolutional neural networks, 2018.
\newblock URL \url{https://arxiv.org/abs/1805.10769}.

\bibitem[Zhou et~al.(2022{\natexlab{a}})Zhou, Li, Ding, You, Qu, and
  Zhu]{Zhou22}
Zhou, J., Li, X., Ding, T., You, C., Qu, Q., and Zhu, Z.
\newblock On the optimization landscape of neural collapse under mse loss:
  Global optimality with unconstrained features, 2022{\natexlab{a}}.
\newblock URL \url{https://arxiv.org/abs/2203.01238}.

\bibitem[Zhou et~al.(2022{\natexlab{b}})Zhou, You, Li, Liu, Liu, Qu, and
  Zhu]{Zhou22b}
Zhou, J., You, C., Li, X., Liu, K., Liu, S., Qu, Q., and Zhu, Z.
\newblock Are all losses created equal: A neural collapse perspective,
  2022{\natexlab{b}}.
\newblock URL \url{https://arxiv.org/abs/2210.02192}.

\bibitem[Zhu et~al.(2021)Zhu, Ding, Zhou, Li, You, Sulam, and Qu]{Zhu21}
Zhu, Z., Ding, T., Zhou, J., Li, X., You, C., Sulam, J., and Qu, Q.
\newblock A geometric analysis of neural collapse with unconstrained features.
\newblock \emph{CoRR}, abs/2105.02375, 2021.
\newblock URL \url{https://arxiv.org/abs/2105.02375}.

\end{thebibliography}
\bibliographystyle{icml2024}

\newpage
\appendix
\onecolumn

\begin{center}
{\bf \Large Appendix for ``Neural Collapse for Cross-entropy Class-Imbalanced Learning with Unconstrained ReLU Features Model''}
\end{center}

\section{Additional Experiments and Network Training details}
\label{sec:additional_exp}

\subsection{Metric definitions}

We define the $\mathcal{NC}$ metrics used in our experiments to measure the discrepancy between the learned model and our derived geometry for the last-layer features and classifier. We recall the notation $\mathbf{h}_{k} \coloneqq \frac{1}{n} \sum_{i=1}^n \mathbf{h}_{k,i}$, i.e., the class-means of class $k$ and $\mathbf{h}_{G} \coloneqq \frac{1}{K n} \sum_{k=1}^K \sum_{i=1}^n \mathbf{h}_{k,i}$ is the feature global-mean. We calculate the within-class covariance matrix $\mathbf{\Sigma}_W \coloneqq \frac{1}{N} \sum_{k=1}^{K} \sum_{i=1}^{n} (\mathbf{h}_{k,i} - \mathbf{h}_{k}) (\mathbf{h}_{k,i} - \mathbf{h}_{k})^{\top}$ and the between-class covariance matrix $\mathbf{\Sigma}_B \coloneqq \frac{1}{K} \sum_{k=1}^K (\mathbf{h}_{k} - \mathbf{h}_{G}) (\mathbf{h}_{k} - \mathbf{h}_{G})^{\top}$. The class-mean matrix is denoted as $\overline{\bh}$. $\{ M_k \}_{k=1}^{K}$ are the constants defined in Eqn. \eqref{eq:M_k} in our main paper.

\noindent \textbf{Feature collapse:}
\begin{align}
\mathcal{NC}1 \coloneqq \frac{1}{K} \text{trace} (\mathbf{\Sigma}_W \mathbf{\Sigma}_B^\dagger), \nonumber
\end{align}
where $\mathbf{\Sigma}_B^\dagger$ is the Moore-Penrose inverse of $\mathbf{\Sigma}_B$.

\noindent  \textbf{Relation between the classifier $\bw$ and features $\bh$}: 
\begin{align}
            &\mathcal{NC}2 - (\bw - \overline{\bh}^{\top}) \coloneqq \left\| \frac{\bw}{\| \bw \|_F} 
            - \frac{\bw_{\text{UFM}_{+}} (\bh) }{\| \bw_{\text{UFM}_{+}} (\bh) \|_F} \right\|_F,
            \text{where } 
            \bw_{\text{UFM}_{+}} (\bh) = 
            \begin{bmatrix}
            K \sqrt{n_1} \bsh_1^{\top} - \sum_{m=1}^{K} \sqrt{n_m}\bsh_m^{\top}    
            \\
            K \sqrt{n_2} \bsh_2^{\top} - \sum_{m=1}^{K} \sqrt{n_m}\bsh_m^{\top} 
            \\
            \ldots
            \\
            K \sqrt{n_K} \bsh_K^{\top} - \sum_{m=1}^{K} \sqrt{n_m}\bsh_m^{\top} \nonumber
            \end{bmatrix}.
\end{align}

\noindent  \textbf{Classifier Gram matrix $\bw \bwt$:}
\begin{align}
    &\mathcal{NC}2 - (\bw \bwt) \coloneqq
    \left\| 
    \frac{\bw \bwt}{\| \bw \bwt \|_F} - \frac{\bw \bwt_{\text{UFM}}}{\| \bw \bwt_{\text{UFM}} \|_F}
    \right\|,
    \nonumber \\
    &\text{where } (\bw \bwt_{\text{UFM}})_{k k}
    = \| \bsw_{k} \|^2 =  \alpha \left(
    (K - 1)^2 \sqrt{n_k} M_k + \sum_{m \neq k}
    \sqrt{n_m} M_m
    \right), 
    \nonumber \\
    &\text{and } (\bw \bwt_{\text{UFM}})_{k j}
    = \bsw_k^{\top} \bsw_j =
        \alpha
        \Bigg[
        - (K - 1)
         \sqrt{n_k} M_k
        - (K - 1)
         \sqrt{n_j} M_j  
         + \sum_{m \neq k, j}
         \sqrt{n_m} M_m \Bigg], \forall \: k \neq j. \nonumber 
\end{align}
The calculation of $\bw \bwt_{\text{UFM}}$ is from the Proposition \ref{prop:classifier_norm}.

\noindent  \textbf{Class-mean Gram matrix $\overline{\bh}^{\top} \overline{\bh}$:}

\begin{align}
    &\mathcal{NC}2 - (\overline{\bh}^{\top} \overline{\bh}) \coloneqq 
    \left\| 
    \frac{\overline{\bh}^{\top} \overline{\bh}}{\| \overline{\bh}^{\top} \overline{\bh} \|_F} 
    - \frac{\overline{\bh}^{\top} \overline{\bh}_{\text{UFM}}}{\| \overline{\bh}^{\top} \overline{\bh}_{\text{UFM}} \|_F}
    \right\|,
    \nonumber \\
    &\text{where } (\overline{\bh}^{\top} \overline{\bh}_{\text{UFM}})_{k k}
    = \| \bsh_{k} \|^2 = \sqrt{\frac{K-1}{K} \frac{\lambda_{W}}{ \lambda_{H}} \frac{1}{n_k}}
    M_k
    \text{ and } (\overline{\bh}^{\top} \overline{\bh}_{\text{UFM}})_{k j}
    = 0. \nonumber 
\end{align}
The calculation of $\overline{\bh}^{\top} \overline{\bh}_{\text{UFM}}$ is from Theorem \ref{thm:CE_main}.

\noindent \textbf{Prediction matrix $\bw \overline{\bh}$:} 
\begin{align}
            &\mathcal{NC}3 - \bw \overline{\bh} \coloneqq \left\| \frac{\bw \bh}{\| \bw \bh \|_{F}} -  \frac{\bw \bh_{\text{UFM}_{+}}}{\| \bw \bh_{\text{UFM}_{+}} \|_{F}} \right\|_{F}, 
            \text{where } \bW \bH_{\text{UFM}_{+} = }
            \begin{bmatrix}
            \frac{K-1}{K} M_1 & \frac{-1}{K} M_2  & ...  & \frac{-1}{K} M_K \\ 
            \frac{-1}{K} M_1 & \frac{K-1}{K} M_2 &  ...& \frac{-1}{K} M_K \\ 
            ... & ... & ... & \\ 
            \frac{-1}{K} M_1 & \frac{-1}{K} M_2  & ... & \frac{K-1}{K} M_K \nonumber
            \end{bmatrix}.
\end{align}

\begin{figure}[t]
\centering
\begin{subfigure}{0.4\textwidth}
    \includegraphics[width=\textwidth]{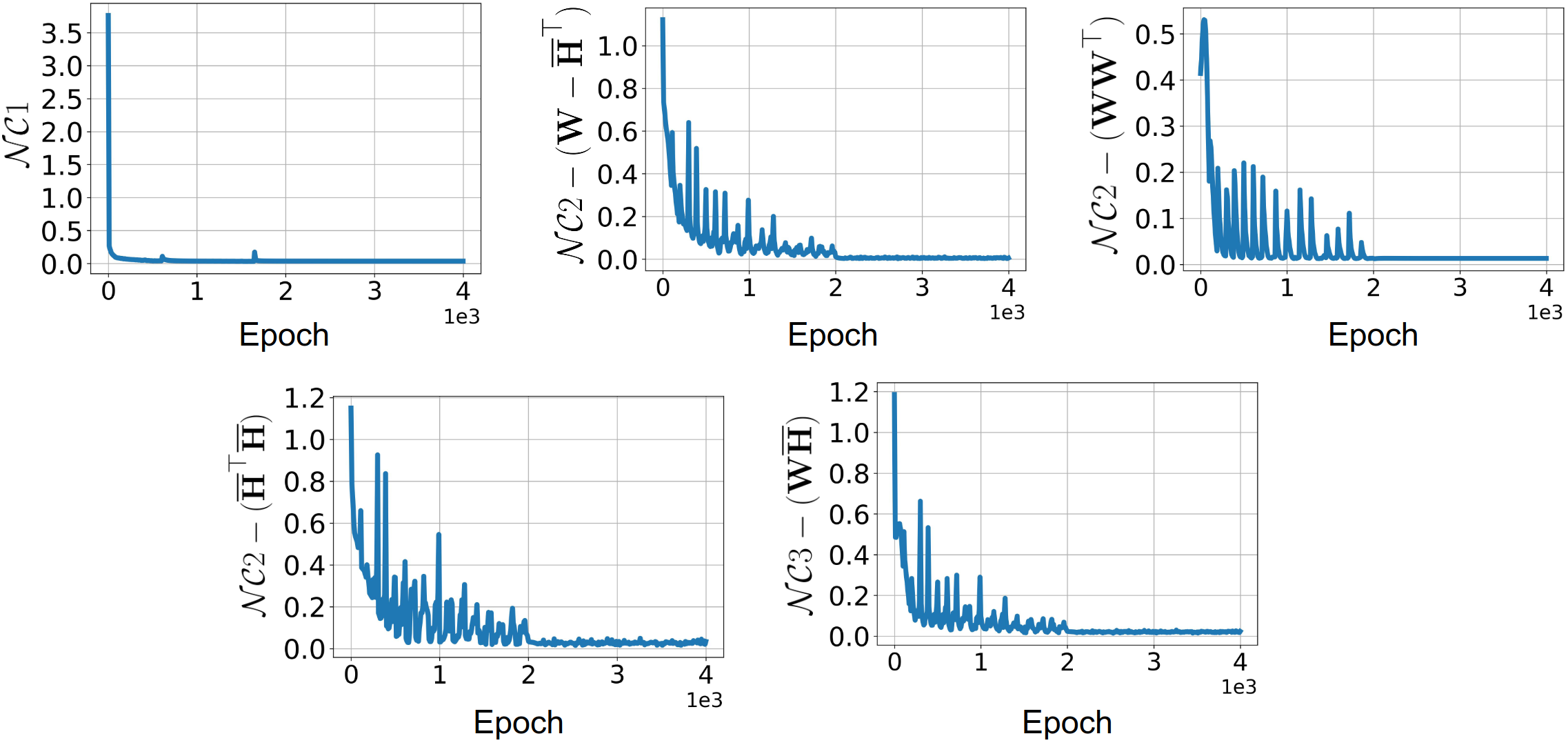}
    \vspace{-0.2in}
    \caption{\small MLP}
    \label{fig:mlp_mnist}
\end{subfigure}
\begin{subfigure}{0.4\textwidth}
    \includegraphics[width=\textwidth]{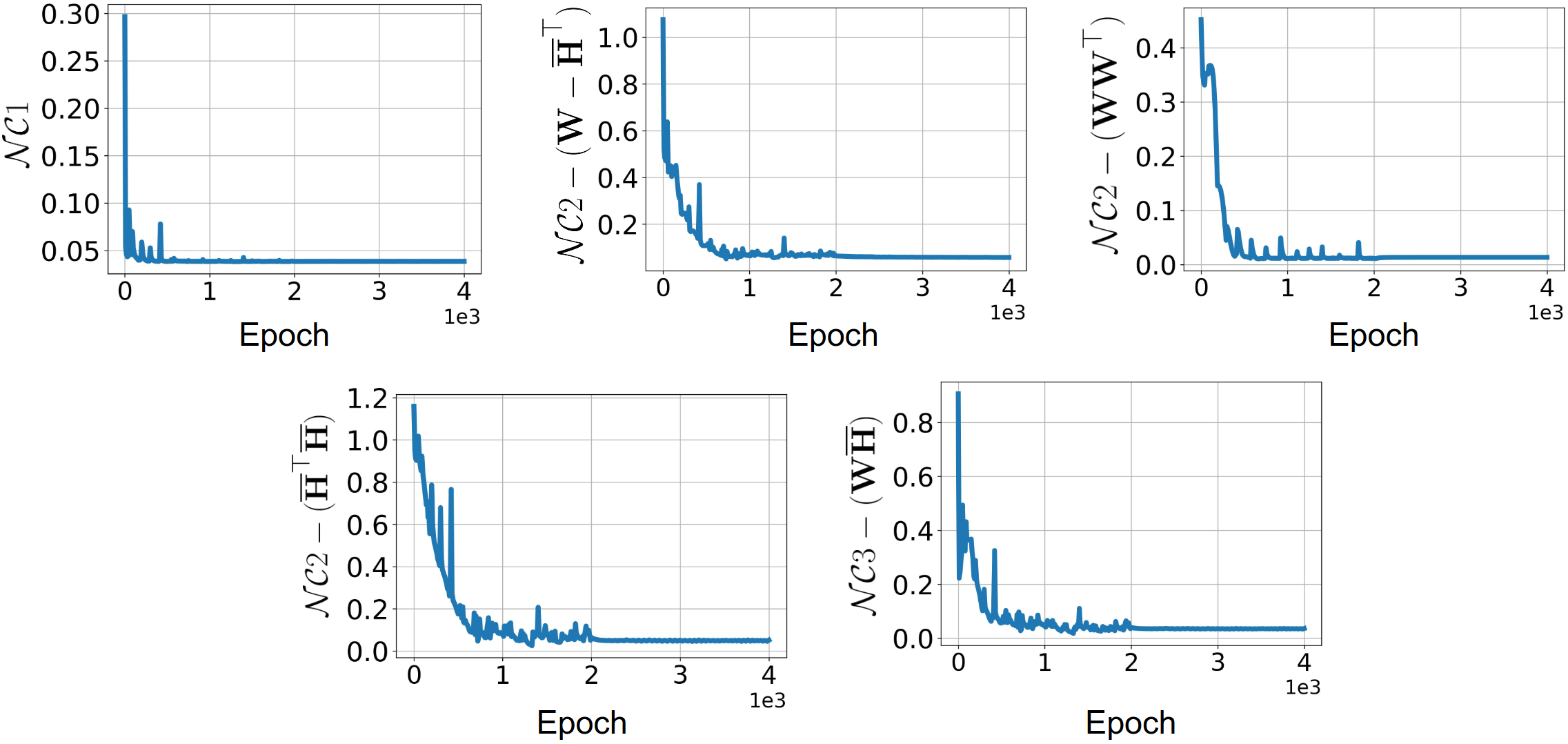}
    \vspace{-0.2in}
    \caption{\small VGG11}
    \label{fig:vgg11_mnist}
\end{subfigure}
\begin{subfigure}{0.4\textwidth}
    \includegraphics[width=\textwidth]{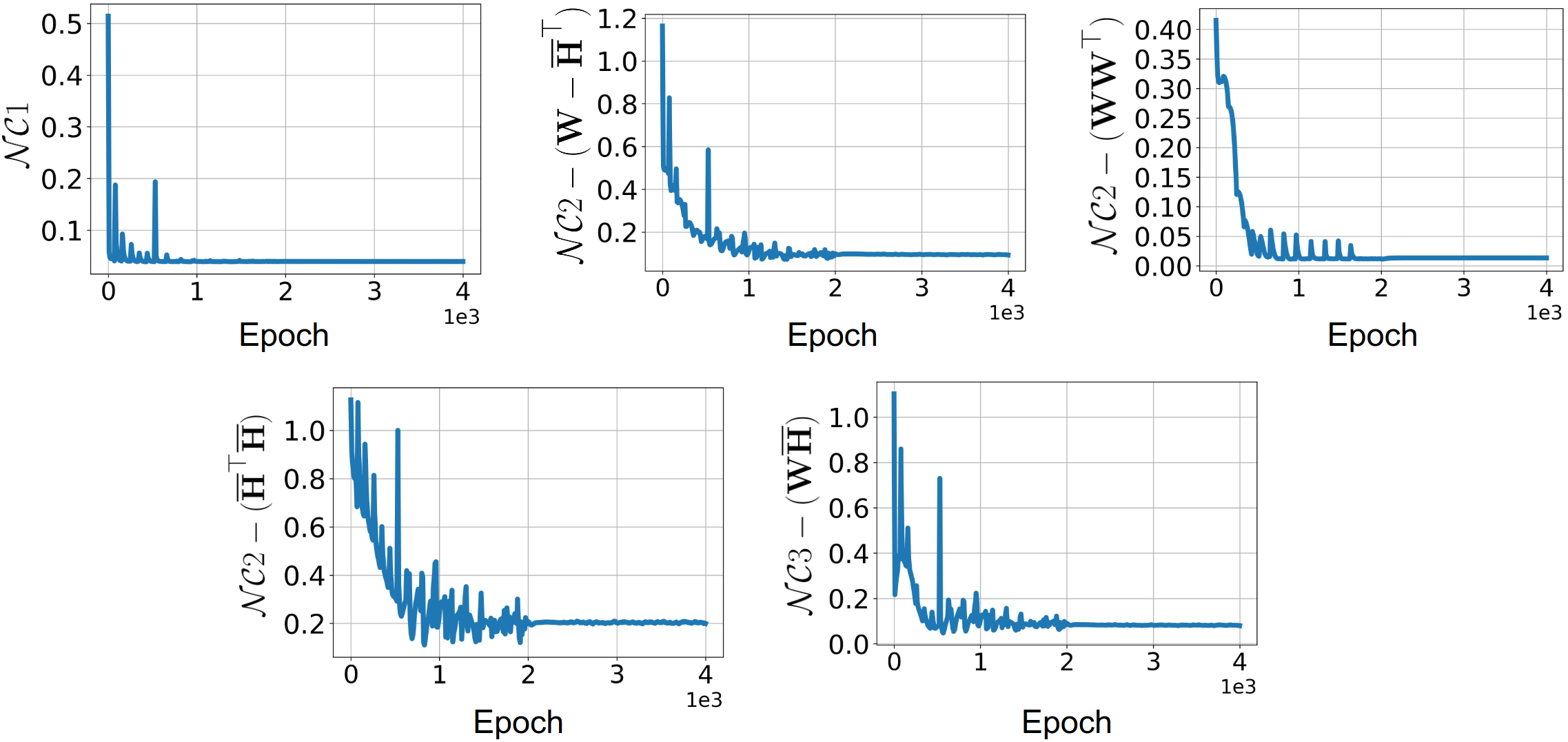}
    \vspace{-0.2in}
    \caption{\small ResNet18}
    \label{fig:resnet18_mnist}
\end{subfigure}
\vspace{-0.1in}
\caption{\small $\mathcal{NC}$ metrics evolution for three models trained on imbalanced subset of MNIST dataset with cross entropy loss.}
\label{fig:mnist_3_models}
\vspace{-0.2in}
\end{figure}

\begin{figure}[h!]
\centering
\begin{subfigure}{0.4\textwidth}
\centering
    \includegraphics[width=\textwidth]{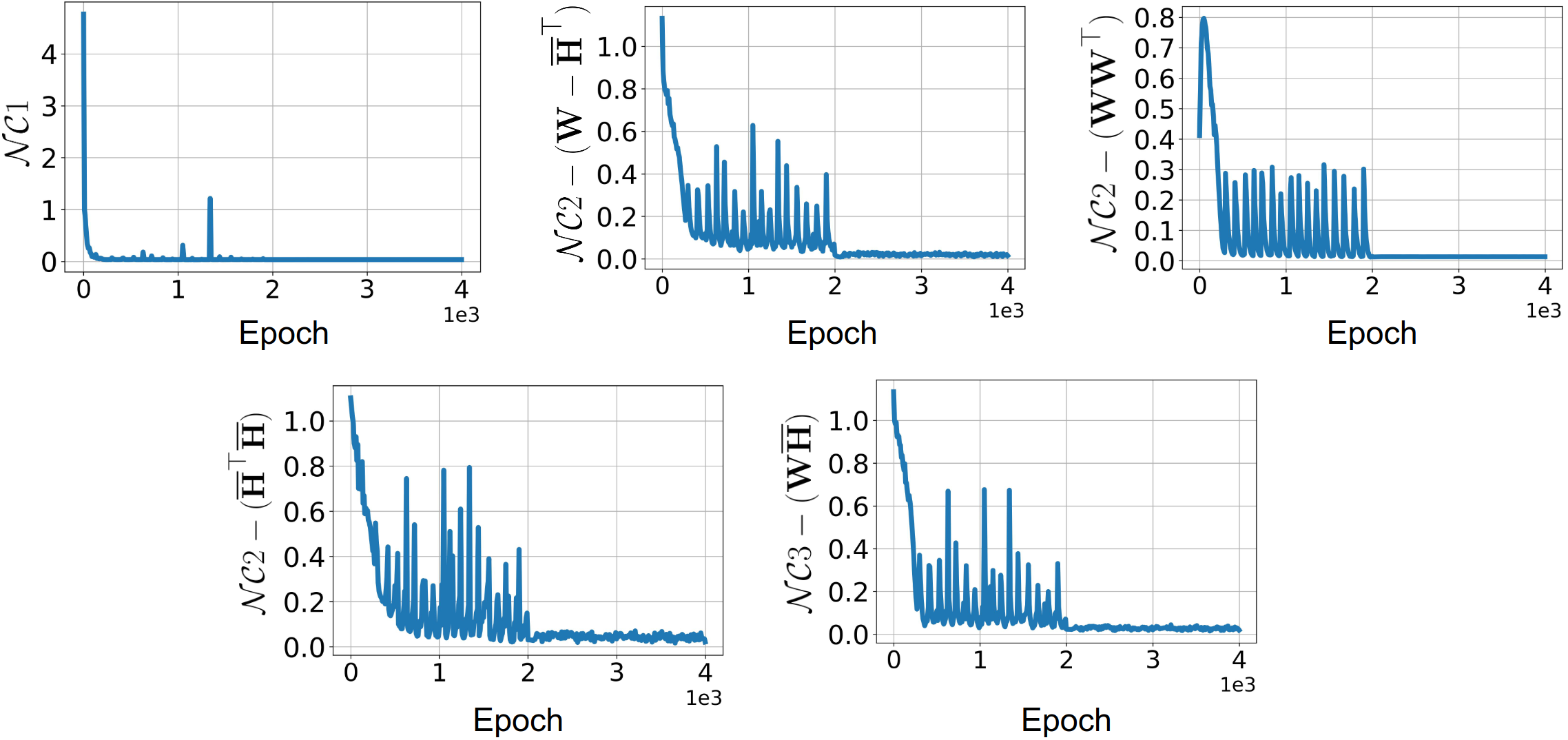}
    \caption{\small MLP}
    \label{fig:mlp_fashion}
\end{subfigure}
\hspace{0.1in}
\begin{subfigure}{0.45\textwidth}
\centering
    \includegraphics[width=\textwidth]{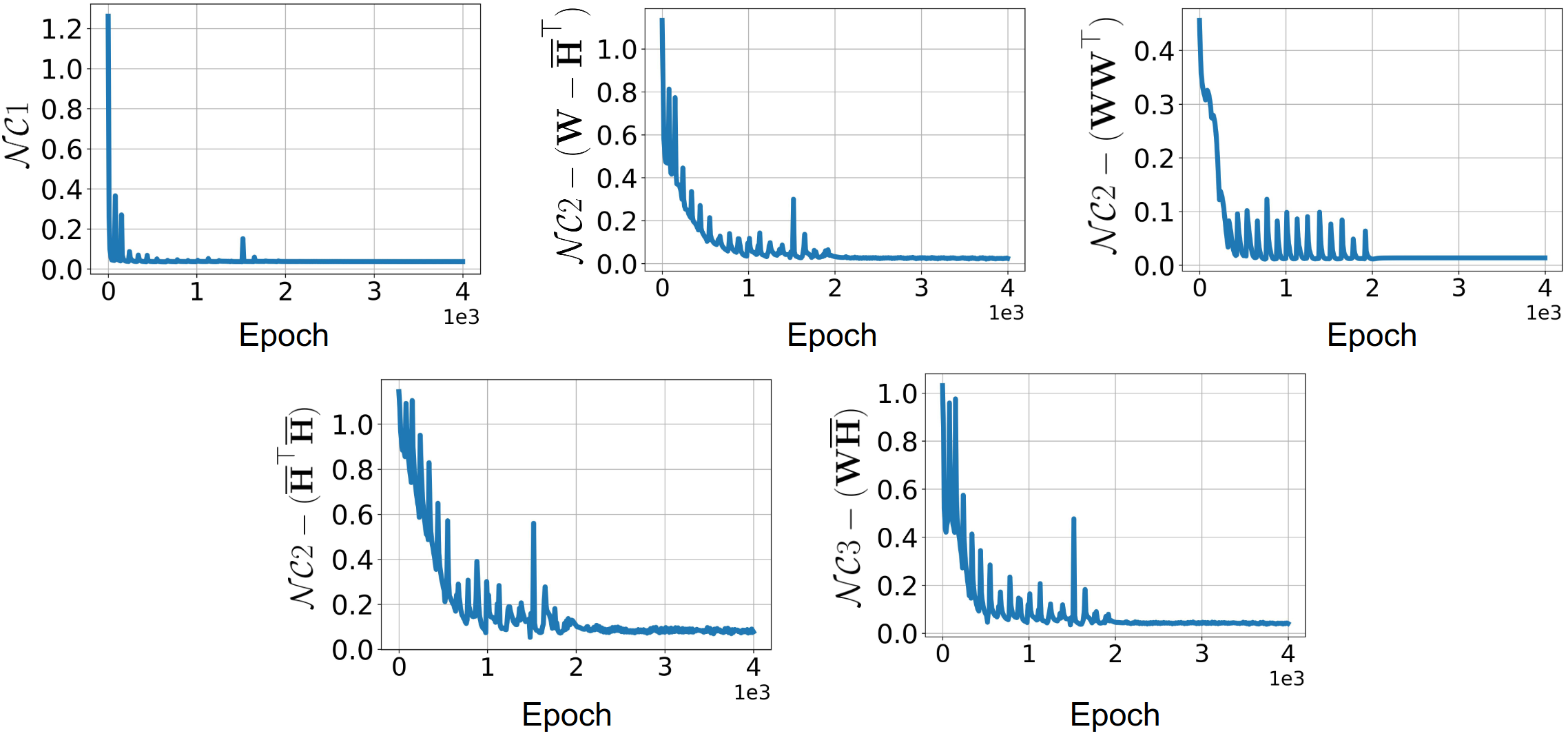}
    \caption{\small VGG11}
    \label{fig:vgg11_fashion}
\end{subfigure}
\begin{subfigure}{0.45\textwidth}
\centering
    \includegraphics[width=\textwidth]{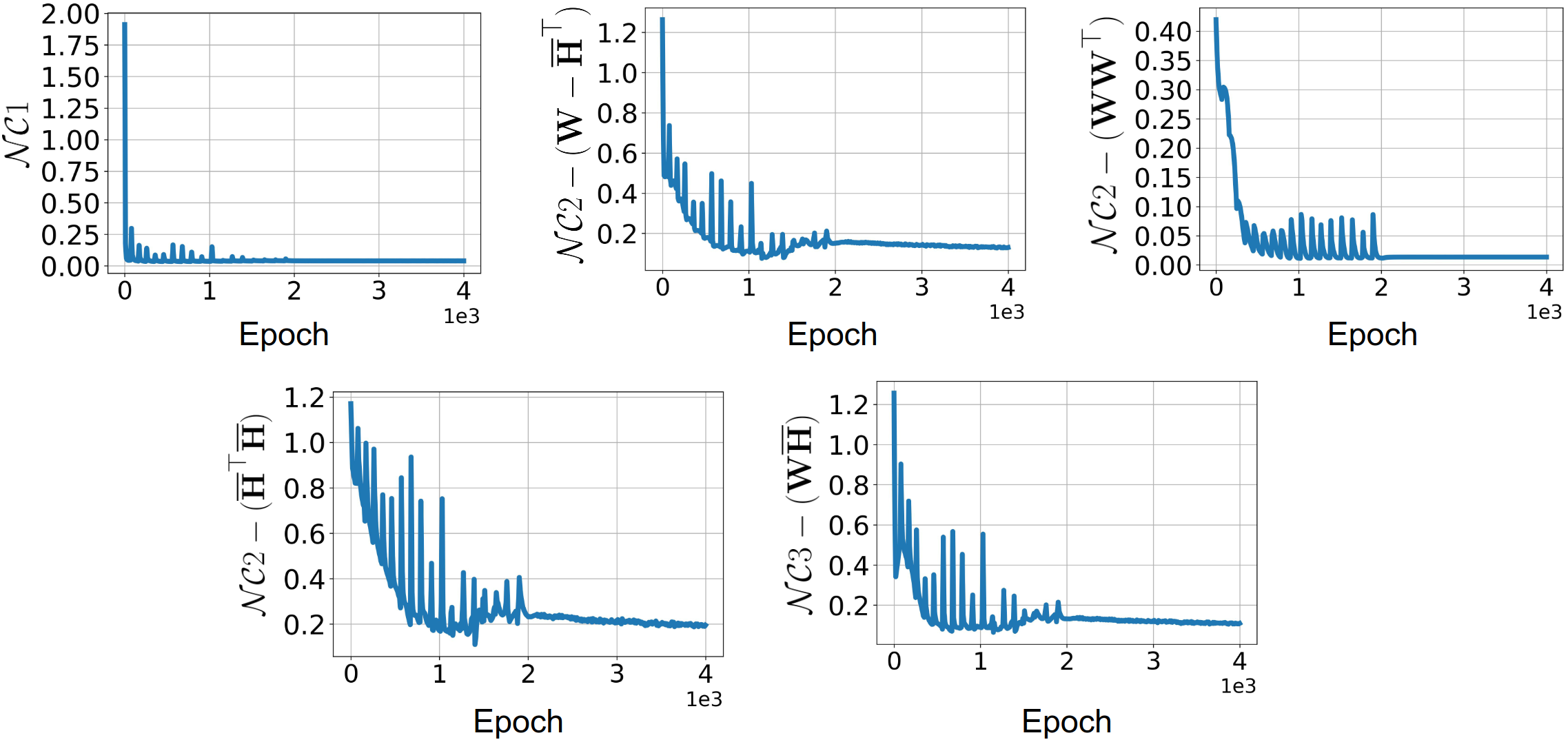}
    \caption{\small ResNet18}
    \label{fig:resnet18_fashion}
\end{subfigure}
\vspace{-0.1in}
\caption{\small $\mathcal{NC}$ metrics evolution for three models trained on imbalanced subset of FashionMNIST dataset with cross entropy loss.}
\label{fig:fashion_3_models}
\vspace{-0.2in}
\end{figure}

\begin{figure}[h!]
\centering
\begin{subfigure}{0.4\textwidth}
    \includegraphics[width=\textwidth]{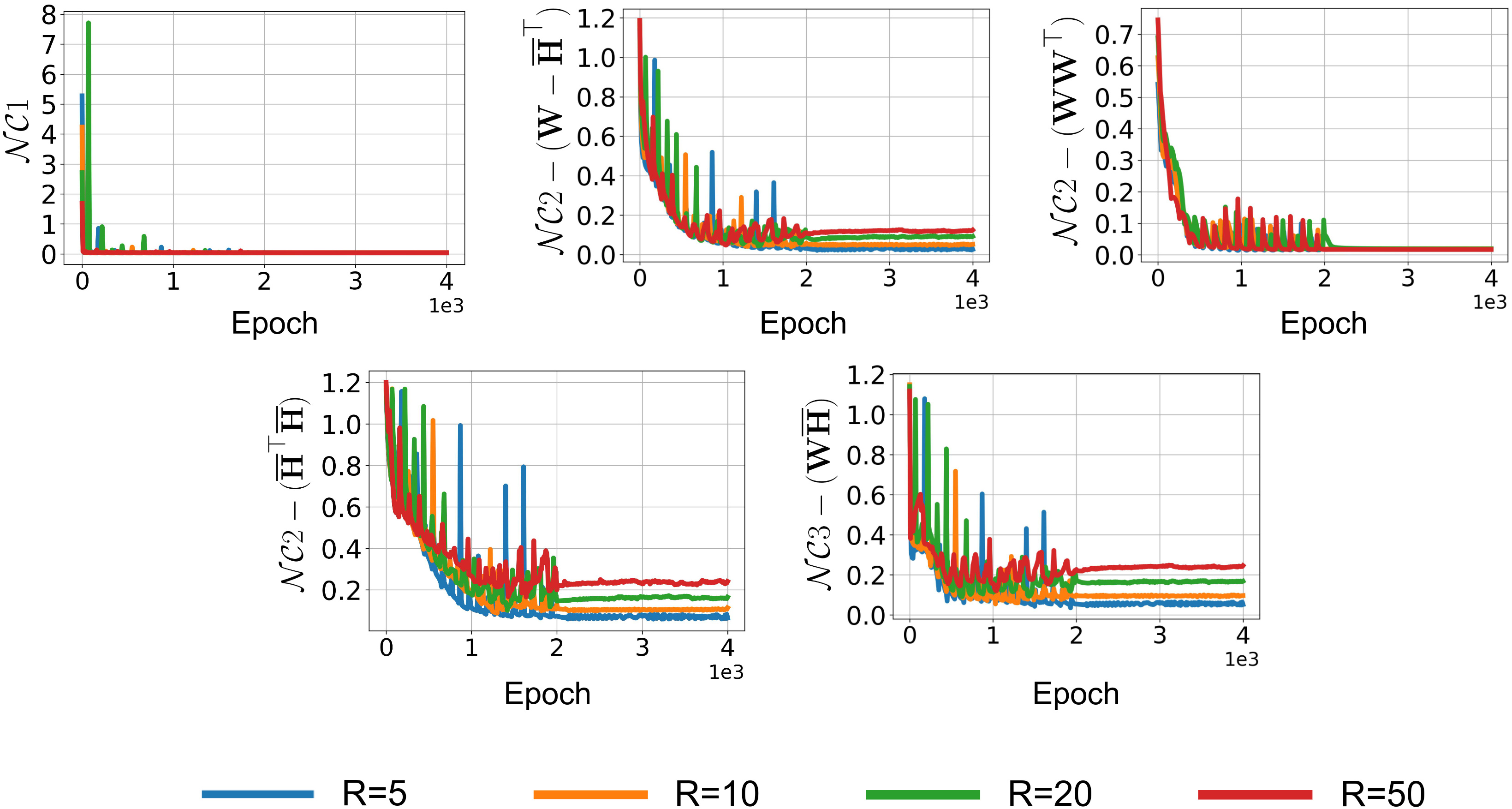}
    \vspace{-0.2in}
    \caption{\small MLP}
    \label{fig:r_mlp_cifar10}
\end{subfigure}
\begin{subfigure}{0.4\textwidth}
    \includegraphics[width=\textwidth]{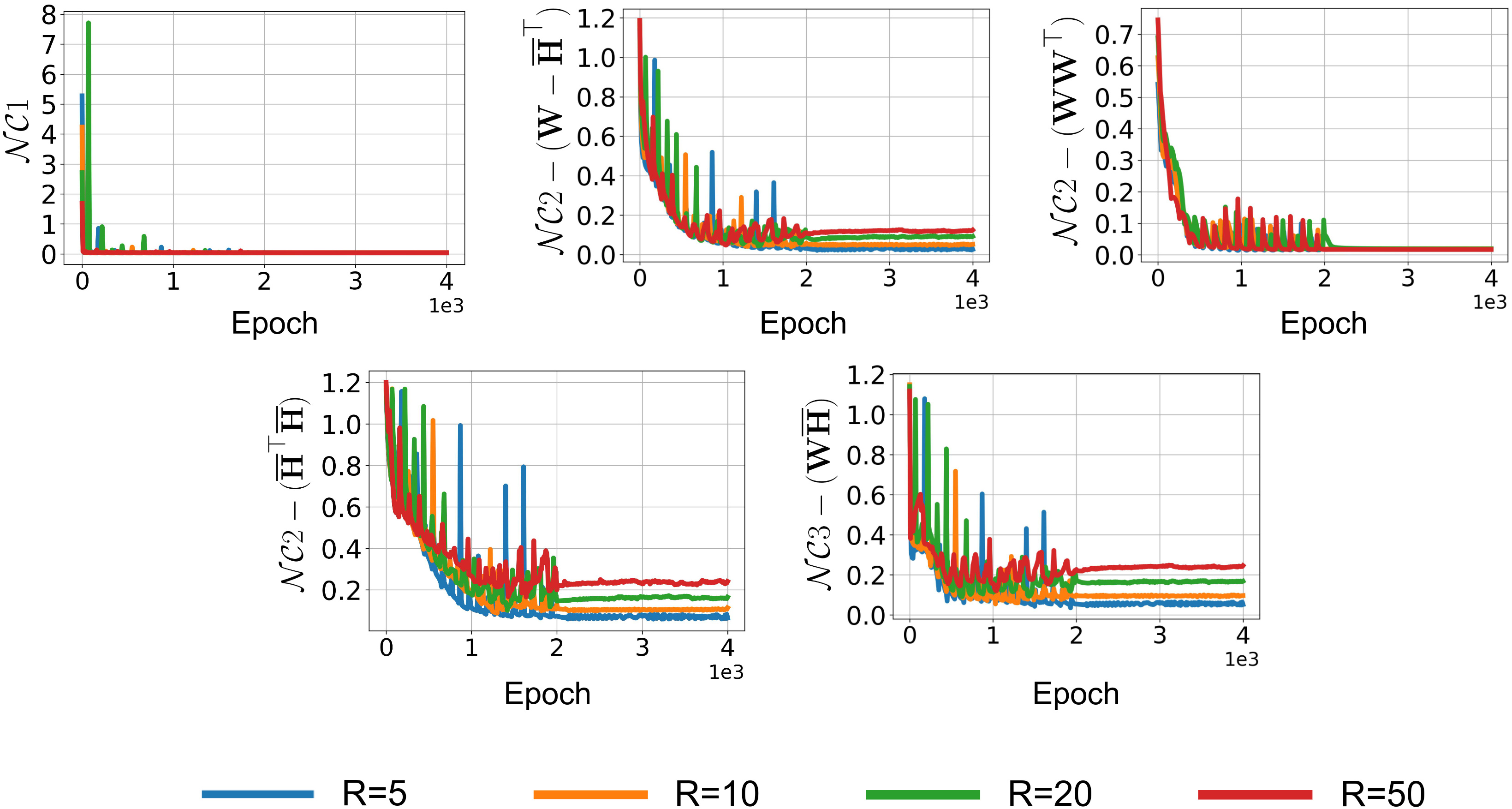}
    \vspace{-0.2in}
    \caption{\small VGG11}
    \label{fig:r_vgg11_cifar10}
\end{subfigure}
\vspace{-0.1in}
\caption{\small $\mathcal{NC}$ metrics evolution of MLP and VGG11 backbone trained on CIFAR10 imbalanced subsets with cross entropy loss}
\label{fig:r_cifar10}
\vspace{-0.2in}
\end{figure}

\subsection{Network training details}

Unless stated otherwise, all models in Section~\ref{subsec:experiment_details} are trained for 4000 epochs with Adam optimizer, we set the general learning rate to $\num{1e-3}$ with decay of $0.1$ at $2000$-th epoch. All MLP models share the same hidden dimension of $1024$.

\textbf{Image classification experiment on CIFAR10:} For this experiment, a subset of the CIFAR10 dataset with $\{1000, 1000, 2000, 2000, 3000, 3000,$ $4000, 4000, 5000, 5000\}$ random samples per class is utilized as training data. We train each backbone model with Adam optimizer with batch size $256$, the weight decay is $\lambda_{W} = \num{1e-4}$. Feature decay $\lambda_{H}$ is set to $\num{1e-5}$ for MLP and VGG11, and to $\num{1e-4}$ for ResNet18. 

\textbf{Image classification experiment on CIFAR100:} We create a random subset of the CIFAR100 dataset with 100 samples per class for the first 20 classes, 200 samples per class for the next 20 classes,..., 500 samples per class for the remaining 20 classes.
Each backbone model is then trained with Adam optimizer with batch size $256$, the learning rate is  $\num{2e-4}$ for VGG11, ResNet18 and $\num{1e-4}$ for MLP. Weight decay $\lambda_{W}$ and feature decay $\lambda_{H}$ is set to $\num{1e-4}$ and $\num{1e-5}$, respectively.

\textbf{Image classification experiment on MNIST dataset:} In this experiment, a randomly sampled subset of MNIST dataset with the number of samples per class $\in \{100, 100, 200, 200, 300, 300,$ $400, 400, 500, 500\}$ is utilized. Each backbone model is trained with batch size $16$. Feature decay rate is $\lambda_{H} = \num{1e-5}$ and weight decay rate is $\lambda_{W} = \num{1e-4}$. The results are shown in Figure \ref{fig:mnist_3_models}.

\textbf{Image classification experiment on FashionMNIST dataset:} Similar to MNIST experiment, we randomly sample a subset of FashionMNIST dataset with $\{100, 100, 200, 200, 300, 300,$ $400, 400, 500, 500\}$ samples per class. Each backbone model is trained with batch size $16$. Feature decay rate is $\lambda_{H} = \num{1e-5}$ and weight decay rate is $\lambda_{W} = \num{1e-4}$. The results are shown in Figure \ref{fig:fashion_3_models}.



\textbf{Varying imbalance ratio:}  Each model is trained with a batch size of 32, weight decay $\lambda_{W}$ of $\num{1e-4}$, and feature decay $\lambda_{H}$ of $\num{1e-5}$. The training data is randomly drawn from CIFAR10 and MNIST dataset with 5 majority classes with 500 samples per class, and the other 5 minority classes with $500 / R$ ($R = 5, 10, 20, 50$) samples per class.

\section{Proof of Theorem \ref{thm:CE_main} and Proposition \ref{prop:classifier_norm}}
\label{sec:proof}
\noindent Recall the training problem:
\begin{align}
    &\min_{\mathbf{W}, \mathbf{H}} 
    \mathcal{L}_{0} (\bw, \bh)
    := 
    \frac{1}{N} \sum_{k=1}^{K} \sum_{i=1}^{n_k} 
    \mathcal{L}_{CE} (\mathbf{W} \mathbf{h}_{k,i}, \mathbf{y}_{k}) 
    +
    \frac{\lambda_{W}}{2} \| \mathbf{W} \|^2_F + \frac{\lambda_{H}}{2} \| \mathbf{H} \|^2_F,
\end{align}
where $\mathbf{h}_{k,i} \geq 0 \: \forall \: k, i$ and:
\begin{align}
    \mathcal{L}_{CE} (\mathbf{z}, \mathbf{y}_{k}) 
    := -\log \left(
    \frac{ e^{z_{k}} }{ \sum_{m=1}^{K} e^{z_{m}}}
    \right).
    \nonumber
\end{align}

\noindent Denoting the class-mean of $k$-th class as $\bsh_k = \frac{1}{n_k} \sum_{i=1}^{n_k} \bsh_{k,i}$, the $m$-th row vector of $\bw$ as $\bsw_{m}$. We denote $\bsz_{k,i} := \bw \bsh_{k,i}$ and $\bsz^{(m)}$ is the $m$-th component of vector $\bsz$.
\\

\noindent \textbf{Step 1:} We introduce a lower bound on the loss $\mathcal{L}_{0}$ by grouping the cross-entropy term and regularization term for features within the same class. 
\\

\noindent We have:
\begin{align}
\begin{aligned}
    &\mathcal{L}_{0} (\bw, \bh) = \frac{1}{N} \sum_{k=1}^{K} \sum_{i=1}^{n_k} 
    \mathcal{L}_{CE} (\bsz_{k,i}, \mathbf{y}_{k}) 
    +
    \frac{\lambda_{W}}{2} \| \mathbf{W} \|^2_F + \frac{\lambda_{H}}{2} \| \mathbf{H} \|^2_F \nonumber \\
    &= 
    \frac{1}{N} \sum_{k=1}^{K}
    \left( \sum_{i=1}^{n_k} 
    \log 
    \left( \frac{ \sum_{m = 1}^K \exp \left( \bsz_{k,i}^{(m)} \right)}{\exp \left( \bsz_{k,i}^{(k)} \right)} \right) \right) 
    + 
    \frac{\lambda_{W}}{2} \| \mathbf{W} \|^2_F 
    + \frac{\lambda_{H}}{2} \sum_{k=1}^{K} \left( \sum_{i=1}^{n_k} \| \bsh_{k,i} \|^2 \right)
    \\ 
    &=
    \frac{1}{N} \sum_{k=1}^{K}
    \left( \sum_{i=1}^{n_k} 
    \log 
    \left( 1 + \sum_{m \neq k} \exp \left( \bsz_{k,i}^{(m)} - \bsz_{k,i}^{(k)} \right) \right) \right) 
    + 
    \frac{\lambda_{W}}{2} \| \mathbf{W} \|^2_F 
    + \frac{\lambda_{H}}{2} \sum_{k=1}^{K} \left( \sum_{i=1}^{n_k} \| \bsh_{k,i} \|^2 \right)
    \\
    &\geq 
     \frac{1}{N} \sum_{k=1}^{K}
    \left( \sum_{i=1}^{n_k} 
    \log 
    \left( 1 + (K - 1) \exp \left(
    \sum_{m \neq k} 
    \frac{\bsz_{k,i}^{(m)} - \bsz_{k,i}^{(k)}}{K-1}
    \right) \right) \right) 
    + 
    \frac{\lambda_{W}}{2} \| \mathbf{W} \|^2_F 
    + \frac{\lambda_{H}}{2} \sum_{k=1}^{K} \left( \sum_{i=1}^{n_k} \| \bsh_{k,i} \|^2 \right) 
    \\
    &=
    \frac{1}{N} \sum_{k=1}^{K}
    \left( \sum_{i=1}^{n_k} 
    \log 
    \left( 1 + (K - 1) \exp \left(
    \frac{ \sum_{m =1}^{K} \bsz_{k,i}^{(m)} - K \bsz_{k,i}^{(k)}}{K-1}
    \right) \right) \right) 
    + 
    \frac{\lambda_{W}}{2} \| \mathbf{W} \|^2_F 
    + \frac{\lambda_{H}}{2} \sum_{k=1}^{K} \left( \sum_{i=1}^{n_k} \| \bsh_{k,i} \|^2 \right) 
    \\
    &\geq 
    \frac{1}{N} \sum_{k=1}^{K}
    n_k  
    \log 
    \left( 1 + (K - 1) \exp \left(
    \frac{1}{n_k}
    \sum_{i=1}^{n_k}
    \frac{ \sum_{m =1}^{K} \bsz_{k,i}^{(m)} - K \bsz_{k,i}^{(k)}}{K-1}
    \right) \right) 
    + 
    \frac{\lambda_{W}}{2} \| \mathbf{W} \|^2_F 
    + \frac{\lambda_{H}}{2} \sum_{k=1}^{K} \left( 
    \frac{1}{n_k} \left\| \sum_{i=1}^{n_k} \bsh_{k,i}  \right\|^2
    \right) 
    \\
    &= 
    \frac{1}{N} \sum_{k=1}^{K}
    n_k  
    \log 
    \left( 1 + (K - 1) \exp \left(
    \frac{ \sum_{m =1}^{K} \bsz_{k}^{(m)} - K \bsz_{k}^{(k)}}{K-1}
    \right) \right)  
    + 
    \frac{\lambda_{W}}{2} \| \mathbf{W} \|^2_F 
    + \frac{\lambda_{H}}{2} \sum_{k=1}^{K}  
    n_k \left\|  \bsh_{k}  \right\|^2
    \\
    &= 
    \frac{1}{N} \sum_{k=1}^{K}
    n_k  
    \log 
    \left( 1 + (K - 1) \exp \left(
    \frac{ \sum_{m =1}^{K} \bsw_{m} \bsh_{k} - K \bsw_{k} \bsh_{k}}{K-1}
    \right) \right)  
    + 
    \frac{\lambda_{W}}{2} \| \mathbf{W} \|^2_F 
    + \frac{\lambda_{H}}{2} \sum_{k=1}^{K}  
    n_k \left\|  \bsh_{k}  \right\|^2
    \\
    &:= \mathcal{L}_{1} (\bw, \bh)
\end{aligned}
\end{align}
where $\bsz_{k} := \frac{1}{n_k} \sum_{i=1}^{n_k} \bsz_{k,i}$. We denote the function:
\begin{align}
    g(\bw \bh) := \frac{1}{N} \sum_{k=1}^{K}
    n_k  
    \log 
    \left( 1 + (K - 1) \exp \left(
    \frac{ \sum_{m =1}^{K} \bsw_{m} \bsh_{k} - K \bsw_{k} \bsh_{k}}{K-1}
    \right) \right),  
\end{align}
thus, $\mathcal{L}_{1}(\bw, \bh) = g(\bw \bh) + \frac{\lambda_{W}}{2} \| \mathbf{W} \|^2_F 
+ \frac{\lambda_{H}}{2} \sum_{k=1}^{K}  
n_k \left\|  \bsh_{k}  \right\|^2$.
\\

\noindent The first inequality above follows from Jensen inequality that:
\begin{align}
    \sum_{m \neq k} \exp \left( \bsz_{k,i}^{(m)} - \bsz_{k,i}^{(k)} \right)
    \geq 
    (K - 1) \exp \left(
    \sum_{m \neq k} 
    \frac{\bsz_{k,i}^{(m)} - \bsz_{k,i}^{(k)}}{K-1}
    \right),
\end{align}
which become equality when and only when $\bsz_{k,i}^{(m)} = \bsz_{k,i}^{(l)} \: \forall \: m, l \neq k$.  The second inequality includes two inequalities as following. The first one is
\begin{align}
    \sum_{i=1}^{n_k} 
    \log 
    \left( 1 + (K - 1) \exp \left(
    \frac{ \sum_{m =1}^{K} \bsz_{k,i}^{(m)} - K \bsz_{k,i}^{(k)}}{K-1}
    \right) \right)
    \geq 
    n_k  
    \log 
    \left( 1 + (K - 1) \exp \left(
    \frac{1}{n_k}
    \sum_{i=1}^{n_k}
    \frac{ \sum_{m =1}^{K} \bsz_{k,i}^{(m)} - K \bsz_{k,i}^{(k)}}{K-1}
    \right) \right),
    \label{eq:jensen}
\end{align}
where we use Jensen inequality since the function $\log(1 + (K-1) \exp(x))$ is a convex function. The second sub-inequality is that for any $k \in [K]$, $\sum_{i=1}^{n_k} \| \bsh_{k,i} \|^2 \geq \frac{1}{n_k} \| \sum_{k=1}^{K} \bsh_{k,i} \|^2$, which achieves equality if and only if $\bsh_{k,i} = \bsh_{k,j} \: \forall \: i \neq j$. 
This equality condition, $\bsh_{k,i} = \bsh_{k,j} \: \forall \: i \neq j$, also satisfies the equality condition of the inequality \eqref{eq:jensen}, hence we only need to satisfy this property to achieve equality.
\\

\noindent \textbf{Step 2:} We further lower bound the $\log$ term of $\mathcal{L}_{1}$, the idea of this bound is inspired from Lemma D.5 in \cite{Zhu21}
\\

\noindent For any $k \in [K]$ and any $t_k > 0$, we have:
\begin{align}
\begin{aligned}
& \log \left( 1 + (K - 1) \exp \left(
    \frac{ \sum_{m =1}^{K} \bsz_{k}^{(m)} - K \bsz_{k}^{(k)}}{K-1}
    \right) \right)   \\
=& \log \left(\frac{t_k}{1 + t_k} \frac{1 + t_k}{t_k}+\frac{1}{1 + t_k}\left(1 + t_k \right)(K-1) \exp \left(
    \frac{ \sum_{m =1}^{K} \bsz_{k}^{(m)} - K \bsz_{k}^{(k)}}{K-1}
    \right) \right) \\
\geq & \frac{1}{1 + t_k} \log \left(\left(1 + t_k \right) (K-1) \exp \left(
    \frac{ \sum_{m =1}^{K} \bsz_{k}^{(m)} - K \bsz_{k}^{(k)}}{K-1}
    \right) \right) + \frac{t_k}{1 + t_k} \log \left(\frac{1 + t_k}{t_k} \right) 
\\
= & \frac{1}{1 + t_k} \frac{ \sum_{m =1}^{K} \bsz_{k}^{(m)} - K \bsz_{k}^{(k)}}{K-1} + \frac{1}{1 + t_k} \log \left(\left( 1 + t_k \right)(K-1) \right) 
+ \frac{t_k}{1 + t_k} \log \left(\frac{ 1 + t_k}{t_k}\right) 
\\
= & \underbrace{\frac{\sqrt{n_k}}{1 + t_k}}_{c_{1,k}} \sqrt{\frac{1}{n_k}} \frac{ \sum_{m =1}^{K} \bsz_{k}^{(m)} - K \bsz_{k}^{(k)}}{K-1} + \underbrace{\frac{1}{1 + t_k} \log \left(\left( 1 + t_k \right)(K-1)\right) + \frac{t_k}{1 + t_k} \log \left(\frac{ 1 + t_k}{t_k}\right)}_{c_{2,k}} \\
= & \frac{c_{1,k}}{K-1} \frac{ \sum_{m =1}^{K} \bsz_{k}^{(m)} - K \bsz_{k}^{(k)}}{\sqrt{n_k}} + c_{2,k},
\end{aligned}
\end{align}
where the inequality above is from the concavity of the $\log(x)$ function, i.e., $\log(tx + (1-t)y) \geq t\log(x) + (1 - t)\log(y)$ for any $x, y$ and $t \in [0, 1]$. The inequality becomes an equality if any only if:
$$
\frac{1 + t_k}{t_k}=\left(1 + t_k \right)(K-1) \exp \left(\frac{ \sum_{m =1}^{K} \bsz_{k}^{(m)} - K \bsz_{k}^{(k)}}{K-1} \right) \quad \text { or } \quad t_k = 0, \quad \text { or } \quad t_k = +\infty .
$$
However, when $t_k = 0$ or $t_k = +\infty$, the equality is trivial. Therefore, we have:
$$
t_k = \left[(K-1) \exp \left(\frac{ \sum_{m =1}^{K} \bsz_{k}^{(m)} - K \bsz_{k}^{(k)}}{K-1} \right)\right]^{-1}.
$$

\noindent To summary, at this step, we have that for any $k \in [K]$ and any $t_k > 0$:
\begin{align}
    \log \left( 1 + (K - 1) \exp \left(
    \frac{ \sum_{m =1}^{K} \bsz_{k}^{(m)} - K \bsz_{k}^{(k)}}{K-1}
    \right) \right)
    \geq 
    \frac{c_{1,k}}{K-1} \frac{ \sum_{m =1}^{K} \bsz_{k}^{(m)} - K \bsz_{k}^{(k)}}{\sqrt{n_k}} + c_{2,k},
    \label{eq:log_bound}
\end{align}
where $c_{1,k} = \sqrt{n_k}/(1 + t_k)$ and $c_{2,k} = \frac{1}{1 + t_k} \log \left(\left( 1 + t_k \right)(K-1)\right) + \frac{t_k}{1 + t_k} \log \left(\frac{ 1 + t_k}{t_k}\right)$. The inequality becomes an equality when:
\begin{align}
    \label{eq:equality_t_k}
    t_k = \left[(K-1) \exp \left(\frac{ \sum_{m =1}^{K} \bsz_{k}^{(m)} - K \bsz_{k}^{(k)}}{K-1} \right)\right]^{-1}.
\end{align}

\noindent \textbf{Step 3:} We apply the result from $\textbf{Step 2}$ and choose the same $c_{1,k}$ for all classes to lower bound $g(\bw \bh)$ w.r.t. the L2-norm of the class-mean $\| \bsh_{k} \|^2 := x_k$.
\\

\noindent By using the inequality \eqref{eq:log_bound} for $\mathbf{z}_{k,i} =  \bw \mathbf{h}_{k,i}$ and choosing the same scalar $c_{1} := c_{1,1} = \ldots = c_{1,k}$ (recall that $c_{1,k}$ can be chosen arbitrarily), we have:
    \begin{align}
    \begin{aligned}
        \label{eq:first_ine}
        &\frac{K-1}{c_{1}} \left[ g(\bw \bh) - \sum_{k=1}^{K} \frac{n_k}{N} c_{2,k} \right] \\
        = \: &\frac{K-1}{c_1} \left[ \frac{1}{N} \sum_{k=1}^{K} n_k   \log 
        \left( 1 + (K - 1) \exp \left(
        \frac{ \sum_{m =1}^{K} \bsw_{m} \bsh_{k} - K \bsw_{k} \bsh_{k}}{K-1}
        \right) \right)  - \sum_{k=1}^{K} \frac{n_k}{N} c_{2,k} 
        \right] \\
        \geq \: &\frac{1}{N} \sum_{k=1}^{K} n_k 
        \sqrt{\frac{1}{n_k}}
        \left[ 
        \sum_{m=1}^{K} \bsw_{m} \mathbf{h}_{k} - K \bsw_{k} \mathbf{h}_{k} \right] \\
        = \: &\frac{1}{N} \sum_{m=1}^{K} \bsw_{m} \left( \sum_{k=1}^{K} \sqrt{n_k} \bsh_{k} - K \sqrt{n_m} \bsh_{m} \right).
    \end{aligned}
    \end{align}
    
\noindent We know that from the Cauchy-Schwarz inequality for inner product that for any $\mathbf{u}, \mathbf{v} \in \mathbb{R}^{K}$ and any $c_{3} > 0$,
    \begin{align}
        \mathbf{u}^{\top} \mathbf{v} \geq - \frac{c_{3}}{2} \| \mathbf{u} \|_2^2 - \frac{1}{2 c_{3}} \| \mathbf{v} \|_2^2. \nonumber
    \end{align}

\noindent The equality holds when $c_{3} \mathbf{u} = -\mathbf{v}$. Therefore, by applying this inequality for each term  $ \bsw_{m} \left( \sum_{k=1}^{K} \sqrt{n_k} \bsh_{k} - K \sqrt{n_m} \bsh_{m} \right)$, we have:
    \begin{align}
    \begin{aligned}
        & \frac{N(K-1)}{c_1} \left[ g(\bw \bh) - \sum_{k=1}^{K} \frac{n_k}{N} c_{2,k} \right] \\
        \geq & - \frac{c_3}{2} \sum_{m=1}^K
        \left
        \| \bsw_{m} \right\|_2^2
        - \frac{1}{2 c_3} \sum_{m=1}^K
         \left\| \sum_{k=1}^{K} \sqrt{n_k} \bsh_{k} - K \sqrt{n_m} \bsh_{m}
        \right\|_2^2 \\
        = & - \frac{c_3}{2} \| \bw \|_F^2
        - \frac{1}{2 c_3} \sum_{m = 1}^K \| \hat{\bsh}_{m} \|_2^2,
        \label{eq:first_ine}
    \end{aligned}
    \end{align}
    where we denote $\hat{\bsh}_{m} := \sum_{k=1}^{K} \sqrt{n_k} \bsh_{k} - K \sqrt{n_m} \bsh_{m}, \forall \: m \in [K]$, and the above inequality becomes an equality if and only if:
    \begin{align}
        c_{3} \bsw_{m} = - \sum_{k=1}^{K} \sqrt{n_k} \bsh_{k} + K \sqrt{n_m} \bsh_{m}, \: \forall m \in [K]
        \label{eq:c_3_equality}
    \end{align}
    
    \noindent We further have:
    \begin{align}
        \label{eq: hat_h}
        \sum_{m=1}^{K} \| \hat{\bsh}_{m} \|^2
        &= \sum_{m=1}^{K} \left\| \sum_{k=1}^{K} \sqrt{n_k} \bsh_{k} - K \sqrt{n_m} \bsh_{m} \right\|^2
        \nonumber \\
        &= \sum_{m=1}^{K} \left( 
        \left\|  \sum_{k=1}^{K} \sqrt{n_k} \bsh_{k}  \right\|^2
        + K^2 \left\|  \sqrt{n_m} \bsh_{m} \right\|^2 
        - 2K \left\langle \sum_{k=1}^{K} \sqrt{n_k} \bsh_{k}, \sqrt{n_m} \bsh_{m} \right\rangle \right) \nonumber \\
        &= K^2 \sum_{m=1}^{K} 
        \left\| \sqrt{n_m} \bsh_{m} \right\|^2
        + K \left\|  \sum_{k=1}^{K} \sqrt{n_k} \bsh_{k}  \right\|^2
        - 2K \left\langle \sum_{k=1}^{K} \sqrt{n_k} \bsh_{k}, \sum_{m=1}^{K} \sqrt{n_m} \bsh_{m} \right\rangle \nonumber \\
        &= K^2 \sum_{m=1}^{K} 
        \left\| \sqrt{n_m} \bsh_{m} \right\|^2 
        - K \left\|  \sum_{k=1}^{K} \sqrt{n_k} \bsh_{k}  \right\|^2
    \end{align}

    \noindent We lower bound the second term of Eqn. \eqref{eq: hat_h} as following:
    \begin{align}
         \left\|  \sum_{k=1}^{K} \sqrt{n_k} \bsh_{k}  \right\|^2
         &= \sum_{k=1}^{K} \left\|  \sqrt{n_k} \bsh_{k}  \right\|^2 
         + \sum_{k,l, k \neq l} \left\langle \sqrt{n_k} \bsh_{k}, \sqrt{n_l} \bsh_{l}   \right\rangle
         \nonumber \\
         &\geq \sum_{k=1}^{K} \left\|  \sqrt{n_k} \bsh_{k}  \right\|^2,
    \end{align}
    where we use the non-negativity of the features and the equality happens iff $\langle \bsh_{k}, \bsh_{l} \rangle = 0, \forall \: k \neq l.$
    \\
    
    \noindent Thus, we have:
    \begin{align}
        \sum_{m=1}^{K} \| \hat{\bsh}_{m} \|^2
        &\leq 
        K(K-1) \sum_{k=1}^{K} n_k \| \bsh_{k} \|^2,
        \label{eq:overline_h_bound}
    \end{align}
    the equality happens iff $\langle \bsh_{k}, \bsh_{l} \rangle = 0, \forall \: k \neq l.$
    \\
    
    \noindent Now, let $x_k := \| \bsh_{k} \|^2$, at critical points of $\mathcal{L}_{1}$, from Lemma \ref{lm:critical_point} , we have:
    \begin{align}
        \| \bw \|_F^2 = \frac{\lambda_{H}}{\lambda_{W}} \sum_{k=1}^{K} n_k x_k.
    \end{align}

    \noindent Hence:
    \begin{align}
        \frac{N (K-1) }{c_1} \left[ g(\bw \bh) - \sum_{k=1}^{K} \frac{n_k}{N} c_{2, k} \right]
         &\geq 
         - \frac{c_3}{2} \frac{\lambda_{H}}{\lambda_{W}} \left( \sum_{k=1}^{K} n_k x_k \right) - \frac{K(K-1)}{2 c_3} \left( \sum_{k=1}^{K} n_k x_k \right)
         \label{eq:pre_c3}
    \end{align}

    \noindent We will choose $c_{3}$ in advance to let the inequality \eqref{eq:pre_c3} hold. From the equality conditions \eqref{eq:c_3_equality} and  \eqref{eq:overline_h_bound}, we can choose $c_3$ as follows:
    \begin{align}
     &c_{3} \bsw_{m} = - \sum_{k=1}^{K} \sqrt{n_k} \bsh_{k} + K \sqrt{n_j} \bsh_{m}, \: \forall m \in [K] \nonumber 
       \\
    \Rightarrow c_3^2 &= \frac{ \sum_{k=1}^{K} \| \hat{\bsh}_{k} \|^2 }{\sum_{k=1}^{K} \| \bsw_k \|^2} =  \frac{K(K-1) \left( \sum_{k=1}^{K} n_k x_k \right)}{\frac{\lambda_{H}}{\lambda_{W}} \left( \sum_{k=1}^{K} n_k x_k \right) } = 
       \frac{\lambda_{W}}{\lambda_{H}} K(K-1)
       \label{eq:c3} 
    \end{align}
    
    \noindent In summary, from \eqref{eq:pre_c3}, we have the lower bound of $g(\bw \bh)$:
    \begin{align}
        g(\bw \bh)
        \geq 
        \frac{-c_1}{N} \sqrt{\frac{\lambda_{H}}{\lambda_{W}} \frac{K}{K-1}}
        \left( \sum_{k=1}^{K} n_k x_k \right)
        + \sum_{k=1}^{K} \frac{n_k}{N} c_{2, k},
        \label{eq:lower_g}
    \end{align}
    for any $c_1 > 0$. The equality conditions of \eqref{eq:lower_g} is derived at Lemma \ref{lemma:equality_g}.
\\

\noindent \textbf{Step 4:} Now, we use the lower bound of $g(\bw \bh)$ above into the bounding of the loss $\mathcal{L}_{1}(\bw, \bh)$ and use the equality conditions from Lemma \ref{lemma:equality_g} to finish the bounding process. 
\\

\noindent Recall that $x_k = \| \bsh_{k} \|^2$, we have at critical points of $\mathcal{L}_{1}$:
\begin{align}
\begin{aligned}
    \mathcal{L}_{1} (\bw, \bh) 
    &= g(\bw \bh) + 
    \frac{\lambda_{W}}{2} \| \mathbf{W} \|^2_F 
    + \frac{\lambda_{H}}{2} \sum_{k=1}^{K}  
    n_k \left\|  \bsh_{k}  \right\|^2
    \nonumber \\
    &\geq \frac{-c_1}{N} \sqrt{\frac{\lambda_{H}}{\lambda_{W}} \frac{K}{K-1}}
        \left( \sum_{k=1}^{K} n_k x_k \right)
        + \sum_{k=1}^{K} \frac{n_k}{N} c_{2, k} + \lambda_{H}
    \sum_{k=1}^{K} n_k x_k \nonumber \\
    &:= \xi(c_1, x_1, \ldots, x_K),
\end{aligned}
\end{align}
for any $c_1 > 0$ ($c_{2,k}$ can be calculated from $c_1$). From Lemma \ref{lemma:equality_g}, we know that the inequality  $\mathcal{L}_{1} (\bw, \bh) \geq \xi(c_1, x_1, \ldots, x_k)$ becomes an equality if and only if:
\begin{align}
\begin{aligned}
        &\bsh_{k}^{\top} \bsh_{l} = 0, \: \forall \: k \neq l \\
        &\bsw_{m}
         = \sqrt{\frac{1}{K (K-1)}}
         \sqrt{\frac{\lambda_{H}}{ \lambda_{W}}} 
         \left( K \sqrt{n_m} \bsh_{m} - \sum_{k=1}^{K} \sqrt{n_k} \bsh_{k}
         \right), 
         \: \forall m \in [K] \\
        &t_{k} = \frac{1}{K-1} \exp \left( \sqrt{\frac{K}{K-1}}  \sqrt{\frac{\lambda_{H}}{ \lambda_{W}}} \sqrt{n_k} 
         \| \bsh_{k} \|^2 \right), \\
        &c_1 = \frac{\sqrt{n_k}} {1 + \frac{1}{K-1} \exp \left( \sqrt{\frac{K}{K-1}}  \sqrt{\frac{\lambda_{H}}{ \lambda_{W}}} \sqrt{n_k} 
         \| \bsh_{k} \|^2 \right)} = \frac{\sqrt{n_l}} {1 + \frac{1}{K-1} \exp \left( \sqrt{\frac{K}{K-1}}  \sqrt{\frac{\lambda_{H}}{ \lambda_{W}}} \sqrt{n_l} 
         \| \bsh_{l} \|^2 \right)}, \: \forall \: k \neq l \nonumber
\end{aligned}
\end{align}

\noindent Next, we will lower bound $\xi(c_1, x_1, \ldots, x_K)$ under these equality conditions for arbitrary values of $x_1, \ldots, x_K$, as following:
\begin{align}
\begin{aligned}
    & \xi(c_1, x_1, \ldots, x_K) \nonumber \\
    &= \frac{-c_1}{N} \sqrt{\frac{\lambda_{H}}{\lambda_{W}} \frac{K}{K-1}}
        \left( \sum_{k=1}^{K} n_k x_k \right)
        + \sum_{k=1}^{K} \frac{n_k}{N} c_{2, k} + 
        \lambda_{H} \sum_{k=1}^{K} n_k x_k
    \\
    &= - \sum_{k=1}^{K} \frac{1}{N}
    \sqrt{\frac{\lambda_{H}}{\lambda_{W}} \frac{K}{K-1}}
    \frac{n_k \sqrt{n_k} x_k}{1 + t_k}
    + \sum_{k=1}^{K} \frac{n_k}{N}
    \left(
    \frac{1}{1 + t_k} \log((K-1)(1 + t_k))
    + \frac{t_k}{1 + t_k} 
    \log \left( \frac{1 + t_k}{t_k} \right)
    \right) + \lambda_{H}
    \sum_{k=1}^{K} n_k x_k.
\end{aligned}
\end{align}

\noindent Due to the separation of the $x_k$'s, we can minimize them individually. Consider the following function, for any $k \in [K]$:
\begin{align}
    g(x) &= - \frac{1}{N}
    \sqrt{\frac{\lambda_{H}}{\lambda_{W}} \frac{K}{K-1}}
    \frac{n_k \sqrt{n_k} x}{1 + t}
    + \frac{n_k}{N}
    \left(
    \frac{1}{1 + t} \log((K-1)(1 + t))
    + \frac{t}{1 + t} 
    \log \left( \frac{1 + t}{t} \right)
    \right) 
    + \lambda_{H}
     n_k x, \quad x \geq 0 
\end{align}
where $t = \frac{1}{K-1} \exp \left( \sqrt{\frac{K}{K-1}}  \sqrt{\frac{\lambda_{H}}{ \lambda_{W}}} \sqrt{n_k} x \right)$.
\\

\noindent We note that:
\begin{align}
\begin{aligned}
    &\frac{1}{1 + t} \log((K-1)(1 + t))
    + \frac{t}{1 + t} 
    \log \left( \frac{1 + t}{t} \right)
    \nonumber \\
    =  &\frac{1}{1 + t} \log((K-1)(1 + t))
    - \frac{1}{1 + t}  \log \left( \frac{1 + t}{t} \right)
    +  \log \left( \frac{1 + t}{t} \right)    \\
    = &\frac{1}{1 + t} \log((K-1)t) 
    +  \log \left( \frac{1 + t}{t} \right)   \\
    = &\frac{ \sqrt{\frac{\lambda_{H}}{\lambda_{W}} \frac{K}{K-1} n_k} x}{1 + t}
     +  \log \left( \frac{1 + t}{t} \right).
\end{aligned}
\end{align}

\noindent Hence:
\begin{align}
    g(x) &= \frac{n_k}{N} \log \left(1 + (K - 1) 
    \exp \left( - \sqrt{\frac{K}{K-1}}  \sqrt{\frac{\lambda_{H}}{ \lambda_{W}}} \sqrt{n_k} x \right) 
    \right) + 
    \lambda_{H}
     n_k x \nonumber \\
     g^{\prime}(x) 
     &= - \frac{n_k}{N}
     \frac{ \sqrt{\frac{K}{K-1}}  \sqrt{\frac{\lambda_{H}}{ \lambda_{W}}} \sqrt{n_k}}{1 + 
     \frac{1}{K-1} 
     \exp \left(\sqrt{\frac{K}{K-1}}  \sqrt{\frac{\lambda_{H}}{ \lambda_{W}}} \sqrt{n_k} x \right) 
     }
     + 
     \lambda_{H} n_k \\
    g^{\prime}(x) &= 0 \Rightarrow 
    c_1 =  \frac{\sqrt{n_k}} {1 + \frac{1}{K-1} \exp \left( \sqrt{\frac{K}{K-1}}  \sqrt{\frac{\lambda_{H}}{ \lambda_{W}}} \sqrt{n_k} x \right)}
    = N \sqrt{\frac{K-1}{K} \lambda_{W} 
    \lambda_{H}}, \\
    &\Rightarrow 
    x^{*} = 
    \sqrt{\frac{K-1}{K} \frac{\lambda_{W}}{ \lambda_{H}} \frac{1}{n_k}}
    \left( \log(K-1) + 
    \log \left(
    \frac{\sqrt{n_k}}{N \sqrt{\frac{K-1}{K} \lambda_{W} 
    \lambda_{H}}}
    - 1 \right)
    \right).
\end{align}

\noindent Since $x = \| \bsh \|^2 \geq 0$, we have that $x^{*} > 0$ if $(K - 1) 
\left(\frac{\sqrt{n_k}}{N \sqrt{\frac{K-1}{K} \lambda_{W} \lambda_{H}}  } - 1 \right) > 1$ or equivalently, $\frac{N}{\sqrt{n_k}} \sqrt{ \lambda_{W} \lambda_{H}} < \sqrt{\frac{K-1}{K}}$. Otherwise, if $\frac{N}{\sqrt{n_k}} \sqrt{ \lambda_{W} \lambda_{H}} \geq \sqrt{\frac{K-1}{K}}$, we have $g^{\prime}(x) > 0 \quad \forall x > 0$ and thus, $x^{*} = 0$.
\\

\noindent In conclusion, we have:
\begin{align}
    \mathcal{L}_{1} (\bw, \bh) = \xi(c_1, x_1, \ldots, x_K) \geq 
    \sum_{k=1}^{K} g(x_k^*) = \text{const} \nonumber
\end{align}

\noindent For any $(\bw, \bh)$ that the equality conditions at Lemma \ref{lemma:equality_g} do not hold, we have that $\mathcal{L}_{1} (\bw, \bh) > \xi \left(c_1 = N \sqrt{\frac{K-1}{K} \lambda_{W} \lambda_{H}}, x_1, \ldots, x_K \right)$ and:
\begin{align}
\begin{aligned}
    &\xi \left(c_1 = N \sqrt{\frac{K-1}{K} \lambda_{W} \lambda_{H}}, x_1, \ldots, x_K \right) \nonumber \\
    = 
    & \frac{-c_1}{N} \sqrt{\frac{\lambda_{H}}{\lambda_{W}} \frac{K}{K-1}}
        \left( \sum_{k=1}^{K} n_k x_k \right)
        + \sum_{k=1}^{K} \frac{n_k}{N} c_{2, k} + 
       \lambda_{H}
    \left( \sum_{k=1}^{K} n_k x_k \right)  \\
    =
    & \sum_{k=1}^{K} \frac{n_k}{N} c_{2, k} \\
    =
    & \sum_{k=1}^{K} \frac{n_k}{N} \left( \frac{1}{1 + t_k} \log((K-1) t_k) 
    +  \log \left( \frac{1 + t_k}{t_k} \right)
    \right) \qquad (\text{with } t_k = \sqrt{n_k}/c_1 - 1)
    \\
    =
    & \sum_{k=1}^{K} g(x^{*}_{k}),
\end{aligned}
\end{align}
hence, $(\bw, \bh)$ is not optimal.
\\

\noindent \textbf{Step 5:} We finish the proof since $\mathcal{L}_{0}(\bw, \bh) \geq \mathcal{L}_{1}(\bw, \bh) \geq \text{const}$ and we study the equality conditions.
\\

\noindent In conclusion, by summarizing all equality conditions, we have that any optimal $(\bw^{*}, \bh^{*})$ of the original training problem obey the following:
\begin{align}
\begin{aligned}
    &\text{i) } \forall \: k \in [K], \bsh_{k,i} = \bsh_{k,j} \quad \forall \: i \neq j \nonumber \\
    &\text{ii) } \bsh_{k}^{\top} \bsh_{l} = 0 \quad \forall \: k \neq l \\
    &\text{iii) } 
    \bsw_{k}
    = \sqrt{\frac{1}{K (K-1)}}
    \sqrt{\frac{\lambda_{H}}{ \lambda_{W}}} 
     \left( K \sqrt{n_k} \bsh_{k} - \sum_{m=1}^{K} \sqrt{n_m} \bsh_{m}
     \right), 
     \: \forall k \in [K] 
     \\
    &\text{and } \sum_{k=1}^{K} \bsw_{k} = \mathbf{0}
     \\
    &\text{iv) } \| \bsh_{k} \|^2 =
      \sqrt{\frac{K-1}{K} \frac{\lambda_{W}}{ \lambda_{H}} \frac{1}{n_k}}
    \log \left( (K - 1) \left(
    \frac{\sqrt{n_k}}{N
    \sqrt{\frac{K-1}{K} \lambda_{W} \lambda_{H}}}
    - 1\right) \right)
    \\
    &\text{v) For } m \neq k, \bsz_k^{(m)} = (\bw \bsh_k)^{(m)} =  
    - \frac{1}{K} \log \left( (K - 1) \left(
    \frac{\sqrt{n_k}}{N
    \sqrt{\frac{K-1}{K} \lambda_{W} \lambda_{H}}}
    - 1\right) \right), \\
    &\qquad 
    \bsz_k^{(k)} = (\bw \bsh_k)^{(k)}
    = \frac{K - 1}{K} 
    \log \left( (K - 1) \left(
    \frac{\sqrt{n_k}}{N
    \sqrt{\frac{K-1}{K} \lambda_{W} \lambda_{H}}}
    - 1\right) \right)
\end{aligned}
\end{align}

\noindent We proceed to deduce the results of Proposition \ref{prop:classifier_norm}:
\begin{align}
    \| \bsw_{k} \|^2 
    &= 
    \frac{1}{K(K-1)} 
    \frac{\lambda_H}{\lambda_W}
    \left\| (K - 1) \sqrt{n_k} \bsh_{k} - \sum_{m \neq k} \sqrt{n_m} \bsh_{m} \right\|^2 \nonumber \\
    &=
    \frac{1}{K(K-1)} 
    \frac{\lambda_H}{\lambda_W}
    \left(
    (K-1)^2 n_k \| \bsh_k \|^2
    + \sum_{m \neq k} 
    n_m \| \bsh_m \|^2
    \right) \nonumber \\
    &= \frac{1}{K \sqrt{K (K - 1)}}
    \sqrt{\frac{\lambda_H}{\lambda_W}}  
    \left(
    (K - 1)^2 \sqrt{n_k} M_k + \sum_{m \neq k}
    \sqrt{n_m} M_m
    \right),
    \nonumber \\
    \bsw_{k}^{\top} \bsw_{j}
    &= 
    \frac{1}{K(K-1)} \frac{\lambda_H}{\lambda_W}
    \left\langle
    (K - 1) \sqrt{n_k} \bsh_{k} - \sum_{m \neq k} \sqrt{n_m} \bsh_{m},
    (K - 1) \sqrt{n_j} \bsh_{j} - \sum_{m \neq j} \sqrt{n_m} \bsh_{m}
    \right\rangle
    \nonumber \\
    &= 
     \frac{1}{K(K-1)} \frac{\lambda_H}{\lambda_W}
     \bigg(
     - (K - 1) n_k \| \bsh_k \|^2
     - (K - 1) n_j \| \bsh_j \|^2
     + \sum_{m \neq k,j} n_m \| \bsh_m \|^2
     \bigg)
     \nonumber \\
     &=  \frac{1}{K \sqrt{K(K-1)}}
    \sqrt{\frac{\lambda_H}{\lambda_W}}
    \Bigg[
    - (K - 1)
     \sqrt{n_k} M_k
    - (K - 1)
     \sqrt{n_j} M_j  + \sum_{m \neq k, j}
     \sqrt{n_m} M_m \Bigg], k \neq j \nonumber
\end{align}

\subsection{Supporting lemmas}

\noindent \textbf{Remark:} Although our training problem is a constrained optimization problem, the constraints $\bsh_{k,i} \geq 0$ are affine functions, it is clear that strong duality holds with dual variables equal $0$'s. Then, the solutions of the primal problem and the optimal dual variables will satisfy KKT conditions and hence, we have $\nabla_{\mathbf{H}} \mathcal{L}_{1} = \mathbf{0}$ at optimal.

\begin{lemma}
    \label{lm:critical_point}
    Any critical points $(\bw, \bh)$ of $\mathcal{L}_{1}(\bw, \bh)$ satisfy:
    \begin{align}
        \| \bw \|_F^2 = \frac{\lambda_{H}}{\lambda_{W}} \sum_{k=1}^{K} n_k \| \bsh_{k} \|^2
    \end{align}
\end{lemma}

\begin{proof}[Proof of Lemma \ref{lm:critical_point}]
    Recall that 
    $\mathcal{L}_{1} (\bw, \bh) = g(\bw \bh) + 
    \frac{\lambda_{W}}{2} \| \mathbf{W} \|^2_F 
    + \frac{\lambda_{H}}{2} \sum_{k=1}^{K}  
    n_k \left\|  \bsh_{k}  \right\|^2$. We have:
    \begin{align}
        &\nabla_{\bw} \mathcal{L}_{1} (\bw, \bh) 
        = \nabla_{\bz = \bw \bh} \: g(\bw \bh) \bh^{\top} + \lambda_{W} \bw = \mathbf{0},
        \nonumber \\
        &\nabla_{\bh} \mathcal{L}_{1} (\bw, \bh) 
        = \bw^{\top} \nabla_{\bz = \bw \bh} \: g(\bw \bh) + \lambda_{H} 
        \begin{bmatrix}
        n_1 \bsh_1 & n_2 \bsh_2 & \ldots & n_K \bsh_{K} 
        \end{bmatrix} = \mathbf{0}.
        \nonumber
    \end{align}

    \noindent From $\mathbf{0} = \bw^{\top} \nabla_{\bw} \mathcal{L}_{1} (\bw, \bh) - \nabla_{\bh} \mathcal{L}_{1} (\bw, \bh) \bh^{\top}$, we have:
    \begin{align}
        \lambda_{W} \bw^{\top} \bw = \lambda_{H}  \begin{bmatrix}
        n_1 \bsh_1 & n_2 \bsh_2 & \ldots & n_K \bsh_{K} 
        \end{bmatrix} \bh^{\top} \nonumber 
    \end{align}
    Hence, by taking the trace of both sides, we have $\| \bw \|_F^2 = \frac{\lambda_{H}}{\lambda_{W}} \sum_{k=1}^{K} n_k \| \bsh_{k} \|^2$.
\end{proof}

\begin{lemma}
    \label{lemma:equality_g}
   The lower bound \eqref{eq:lower_g} is attained for any critical points $(\bw, \bh)$ if and only if the following hold:
    \begin{align}
        &\bsh_{k}^{\top} \bsh_{l} = 0, \: \forall \: k \neq l \\
        &\bsw_{m}
         = \sqrt{\frac{1}{K (K-1)}}
         \sqrt{\frac{\lambda_{H}}{ \lambda_{W}}} 
         \left( K \sqrt{n_m} \bsh_{m} - \sum_{k=1}^{K} \sqrt{n_k} \bsh_{k}
         \right), 
         \: \forall m \in [K] \\
        &t_{k} = \frac{1}{K-1} \exp \left( \sqrt{\frac{K}{K-1}}  \sqrt{\frac{\lambda_{H}}{ \lambda_{W}}} \sqrt{n_k} 
         \| \bsh_{k} \|^2 \right) \\
        &c_1 = \frac{\sqrt{n_k}} {1 + \frac{1}{K-1} \exp \left( \sqrt{\frac{K}{K-1}}  \sqrt{\frac{\lambda_{H}}{ \lambda_{W}}} \sqrt{n_k} 
         \| \bsh_{k} \|^2 \right)} = \frac{\sqrt{n_l}} {1 + \frac{1}{K-1} \exp \left( \sqrt{\frac{K}{K-1}}  \sqrt{\frac{\lambda_{H}}{ \lambda_{W}}} \sqrt{n_l} 
         \| \bsh_{l} \|^2 \right)}, \: \forall \: k \neq l
    \end{align}
\end{lemma}

\begin{proof}[Proof of Lemma \ref{lemma:equality_g}]
    From the proof above, we see that if we want to achieve the lower bound \eqref{eq:pre_c3}, we need that:
    \begin{align}
        \bsh_{k}^{\top} \bsh_{l} = 0 \quad \forall \: k, l \in [K], k \neq l \nonumber,
        \label{eq:equality_H}
    \end{align}
    to achieve equality for the inequality \eqref{eq:overline_h_bound}.
    \\
    
    \noindent We further need the following to obey \eqref{eq:c_3_equality}:
    \begin{align}
         &c_{3} \bsw_{m} = - \left( \sum_{k=1}^{K} \sqrt{n_k} \bsh_{k} - K \sqrt{n_m} \bsh_{m} \right), 
         \: \forall k \in [K] 
         \qquad
         \text{with } c_3 = \sqrt{\frac{\lambda_{W}}{\lambda_{H}} K(K-1)}
         \nonumber \\
         \Rightarrow &\bsw_{m} 
         = \sqrt{\frac{1}{K (K-1)}}
         \sqrt{\frac{\lambda_{H}}{ \lambda_{W}}} 
         \left( K \sqrt{n_m} \bsh_{m} - \sum_{k=1}^{K} \sqrt{n_k} \bsh_{k}
         \right), 
         \: \forall k \in [K]. \nonumber
    \end{align}
    Thus:
    \begin{align}
        \sum_{k=1}^{K} \bsw_{k} = \mathbf{0}
    \end{align}

    \noindent Next, we need the inequality \eqref{eq:first_ine} to hold. Equivalently, this means that the equality condition \eqref{eq:equality_t_k} need to hold. Indeed, for a given $k \in [K]$ and any $m \neq k$:
    \begin{align}
        \bsz_{k}^{(m)} &= \bsw_{m} \bsh_{k}
        \nonumber \\
        &= \sqrt{\frac{1}{K (K-1)}}
         \sqrt{\frac{\lambda_{H}}{ \lambda_{W}}} 
         \left( K \sqrt{n_m} \bsh_{m} - \sum_{l=1}^{K} \sqrt{n_l} \bsh_{l}
         \right) \bsh_{k}
         \nonumber \\
         &= - \sqrt{\frac{1}{K (K-1)}}
         \sqrt{\frac{\lambda_{H}}{ \lambda_{W}}}
         \sqrt{n_k} \| \bsh_{k} \|^2
         \nonumber \\
         \Rightarrow \:  \bsz_{k}^{(m)} &=  \bsz_{k}^{(l)} \quad \forall m, l \neq k
    \end{align}

    \noindent We further have, for any $k \in [K]$:
    \begin{align}
        \sum_{m=1}^{K} \bsz_{k}^{(m)} 
        &= \left( \sum_{m=1}^{K} \bsw_{m} \right) \bsh_{k}
        = 0, \nonumber \\
        K \bsz_{k}^{(k)} 
        &= K \bsw_{k} \bsh_{k} 
        = 
        K \sqrt{\frac{1}{K (K-1)}}
         \sqrt{\frac{\lambda_{H}}{ \lambda_{W}}}  \left( K \sqrt{n_k} \bsh_{k} - \sum_{m=1}^{K} \sqrt{n_m} \bsh_{m}
         \right)
         \bsh_{k} \nonumber 
         \\
         &= \sqrt{K(K-1)}  \sqrt{\frac{\lambda_{H}}{ \lambda_{W}}} \sqrt{n_k} 
         \| \bsh_{k} \|^2
         \nonumber \\
         \Rightarrow t_k 
         &= \left[(K-1) \exp \left(\frac{ \sum_{m =1}^{K} \bsz_{k}^{(m)} - K \bsz_{k}^{(k)}}{K-1} \right)\right]^{-1}
         = 
         \frac{1}{K-1} 
         \exp \left(
        \sqrt{\frac{K}{K-1}}  \sqrt{\frac{\lambda_{H}}{ \lambda_{W}}} \sqrt{n_k} 
         \| \bsh_{k} \|^2
        \right)
    \end{align}

    \noindent Since the scalar $c_{1}$ is chosen to be the same for all $k \in [K]$, we have:
    \begin{align}
        c_{1}
        &= \frac{\sqrt{n_k}}{1 + t_k}
        = 
        \frac{\sqrt{n_k}}{1 +  \frac{1}{K-1} 
         \exp \left(
        \sqrt{\frac{K}{K-1}}  \sqrt{\frac{\lambda_{H}}{ \lambda_{W}}} \sqrt{n_k} 
         \| \bsh_{k} \|^2
        \right)}, \: \forall \: k \in [K]
    \end{align}
\end{proof}

\section{Comparison with SELI geometry}
\label{sec:comparison_SELI}

In this section, we make a comparison between our geometry derived in Theorem \ref{thm:CE_main}, which is the convergence geometry of the UFM Cross-entropy class-imbalance problem with ReLU features, and SELI \cite{Christos22}, the geometry of the UFM SVM class-imbalance problem. Our conclusion is that both the classifier and features of our geometry are different from those of SELI, for both finite ($\lambda > 0$) and vanishing regularization level ($\lambda \rightarrow 0$).
\\

First, we have a useful result that used for subsequent analysis in the vanishing regularization scenario:

\begin{lemma}
    Let $\{ M_i \}_{i=1}^{K}$ be the constants that we have defined in Eqn. \eqref{eq:M_k} in our main paper. We have:
    \begin{align}
    \lim_{\lambda_{W}, \lambda_{H} \rightarrow 0} \frac{M_i}{M_j} = 1 \nonumber
\end{align}
\end{lemma}

\begin{proof}
    This property can be easily proved using L'Hôpital's rule.
\end{proof}

Let $\barh$ is the class-mean matrix. For the prediction matrix, we have from Theorem \ref{thm:CE_main}, point (e):
    \begin{align}
        &\bz = \bw \barh
        = \begin{bmatrix}
        (1 - 1/K)M_1 & -M_2/K & \ldots & -M_K/K \\ 
        - M_1/K & (1 - 1/K)M_2 & \ldots & -M_K/K \\ 
        \vdots & \vdots & \ddots & \vdots \\ 
        - M_1/K & -M_2/K & \ldots &  (1 - 1/K)M_K
        \end{bmatrix} \nonumber \\
        \Rightarrow &\lim_{\lambda_{W}, \lambda_{H} \rightarrow 0} \frac{\bz}{\| \bz \|_F} \propto 
        \begin{bmatrix}
        1 - 1/K & -1/K & \ldots & -1/K \\ 
        - 1/K & 1 - 1/K & \ldots & -1/K \\ 
        \vdots & \vdots & \ddots & \vdots \\ 
        - 1/K & -1/K & \ldots &  1 - 1/K
        \end{bmatrix}. \nonumber
    \end{align}
    
Hence, (i) the prediction matrix with finite $\lambda$'s is different from SELI's prediction matrix due to the multiplication $M_k$ at each column, (ii) in limiting case where the $\lambda$'s converge to 0, our prediction matrix converges to the ETF matrix. SELI structure, after grouping the identical columns of the SEL matrix (see Definition 2 in \cite{Christos22}), is also an ETF. It is proven in \cite{Christos22} that the matrix $\bw \bh$ follows SEL matrix and since the features of the same class converge to their class-mean, we have $\bz = \bw \barh$ follows ETF structure. 
\\

To derive the classifier and class-means Gram matrices, i.e., $\bw \bwt$ and $\barh^{\top} \barh$, we consider the same setting $(R, 1/2)$ as \cite{Christos22} for easier comparison with results derived for SELI at page 24 and 25 in \cite{Christos22}. Specifically, the setting has total $K$ classes, with $K/2$ classes are majority class with $n_A$ samples per class, the other $K/2$ classes are minority with $n_B$ samples per class. The imbalance ratio $\frac{n_{A}}{n_B}$ is $R$. 
\\

For the matrix $\bw \bwt$, using the results in the Proposition \ref{prop:classifier_norm}, we have:
\begin{align}
    &\bw \bwt 
    \propto 
    \begin{bmatrix}
    \sqrt{R} M_{A} \mathbf{I}_{K/2} - \frac{1}{K} \left( \frac{3}{2} \sqrt{R} M_A - \frac{1}{2} M_B  \right) \mathbf{1}_{K/2} \mathbf{1}_{K/2}^{\top}
    & -\frac{1}{2K} (\sqrt{R} M_{A} + M_{B}) \mathbf{1}_{K/2} \mathbf{1}_{K/2}^{\top}
    \\ -\frac{1}{2K} (\sqrt{R} M_{A} + M_{B}) \mathbf{1}_{K/2} \mathbf{1}_{K/2}^{\top}
     & M_{B} \mathbf{I}_{K/2} - \frac{1}{K} \left( \frac{3}{2} M_B - \frac{1}{2} \sqrt{R} M_A \right) \mathbf{1}_{K/2} \mathbf{1}_{K/2}^{\top}
    \end{bmatrix} \nonumber
    \\
    &\Rightarrow \lim_{\lambda_{W}, \lambda_{H} \rightarrow 0} \frac{\bw \bwt}{\| \bw \bwt \|_F}
    \propto 
    \begin{bmatrix}
    \sqrt{R} \mathbf{I}_{K/2} - \frac{1}{K} \left( \frac{3}{2} \sqrt{R} - \frac{1}{2}  \right) \mathbf{1}_{K/2} \mathbf{1}_{K/2}^{\top}
    &
    -\frac{1}{2K} (\sqrt{R} + 1) \mathbf{1}_{K/2} \mathbf{1}_{K/2}^{\top}
    \\
    -\frac{1}{2K} (\sqrt{R} + 1) \mathbf{1}_{K/2} \mathbf{1}_{K/2}^{\top}
    &
    \mathbf{I}_{K/2} - \frac{1}{K} \left( \frac{3}{2} - \frac{1}{2} \sqrt{R} \right) \mathbf{1}_{K/2} \mathbf{1}_{K/2}^{\top}
    \end{bmatrix}, \nonumber
\end{align}
which has the similar block structure as the Gram matrix $\bw \bwt$ in \cite{Christos22}, but the elements are all different with SELI in both finite $\lambda$'s and vanishing $\lambda$'s cases (see page 24 in \cite{Christos22} for the SELI closed-form $\bw \bwt$). 
\\

The "centering" considered in \cite{Christos22} is equivalent to centering the class-mean matrix $\barh$, which is $\barh - \overline{\bsh}_{G} \mathbf{1}_{K}^{\top}$ with $\overline{\bsh}_{G} = \frac{1}{K} \sum_{i=1}^{K} \bsh_{i}$. Regarding the "centering" class-mean matrix, we have:
\begin{align}
    (\barh - \overline{\bsh}_{G} \mathbf{1}_{K}^{\top})^{\top}
    (\barh - \overline{\bsh}_{G} \mathbf{1}_{K}^{\top})
    =
    (\mathbf{I}_{K} - \frac{1}{K} \mathbf{1}_{K} \mathbf{1}_{K}^{\top})^{\top} \barh^{\top} \barh (\mathbf{I}_{K} - \frac{1}{K} \mathbf{1}_{K} \mathbf{1}_{K}^{\top}). \nonumber
\end{align}

From our Theorem \ref{thm:CE_main}, we have: $\barh^{\top} \barh \propto \operatorname{diag} \left( \frac{M_A}{\sqrt{n_A}}, \ldots, \frac{M_A}{\sqrt{n_A}}, \frac{M_B}{\sqrt{n_B}}, \ldots, \frac{M_B}{\sqrt{n_B}} \right)$. 
\\

Thus,
\begin{alignat}{2}
    &(\barh - \overline{\bsh}_{G} \mathbf{1}_{K}^{\top})^{\top}
    (\barh - \overline{\bsh}_{G} \mathbf{1}_{K}^{\top})
    \propto \nonumber \\
    &\qquad \qquad \qquad \qquad \begin{bmatrix}
    \frac{M_{A}}{\sqrt{R}}  \mathbf{I}_{K/2} - \frac{1}{K} \left( \frac{3}{2} \frac{M_{A}}{\sqrt{R}} - \frac{1}{2} M_B  \right) \mathbf{1}_{K/2} \mathbf{1}_{K/2}^{\top}
    & -\frac{1}{2K} \left(\frac{M_{A}}{\sqrt{R}} + M_{B} \right) \mathbf{1}_{K/2} \mathbf{1}_{K/2}^{\top}
    \\ -\frac{1}{2K} \left(\frac{M_{A}}{\sqrt{R}} + M_{B} \right) \mathbf{1}_{K/2} \mathbf{1}_{K/2}^{\top}
     & M_{B} \mathbf{I}_{K/2} - \frac{1}{K} \left( \frac{3}{2} M_B - \frac{1}{2} \frac{M_{A}}{\sqrt{R}}\right) \mathbf{1}_{K/2} \mathbf{1}_{K/2}^{\top}
    \end{bmatrix} \nonumber \\
    &\Rightarrow 
    \lim_{\lambda_{W}, \lambda_{H} \rightarrow 0}
    \frac{(\barh - \overline{\bsh}_{G} \mathbf{1}_{K}^{\top})^{\top}
    (\barh - \overline{\bsh}_{G} \mathbf{1}_{K}^{\top})}{ \| (\barh - \overline{\bsh}_{G} \mathbf{1}_{K}^{\top})^{\top}
    (\barh - \overline{\bsh}_{G} \mathbf{1}_{K}^{\top}) \|_{F}}
    \propto \nonumber \\
    &\qquad \qquad \qquad \qquad
    \begin{bmatrix}
    \frac{1}{\sqrt{R}}  \mathbf{I}_{K/2} - \frac{1}{K} \left( \frac{3}{2\sqrt{R}} - \frac{1}{2} \right) \mathbf{1}_{K/2} \mathbf{1}_{K/2}^{\top}
    & -\frac{1}{2K} \left(\frac{1}{\sqrt{R}} + 1 \right) \mathbf{1}_{K/2} \mathbf{1}_{K/2}^{\top}
    \\ -\frac{1}{2K} \left(\frac{1}{\sqrt{R}} + 1 \right) \mathbf{1}_{K/2} \mathbf{1}_{K/2}^{\top}
     & \mathbf{I}_{K/2} - \frac{1}{K} \left( \frac{3}{2}  - \frac{1}{2 \sqrt{R}}\right) \mathbf{1}_{K/2} \mathbf{1}_{K/2}^{\top}
    \end{bmatrix}, \nonumber
\end{alignat}
which again has the similar block structure as the matrix $\bh^{\top} \bh$ in \cite{Christos22}, but the elements are all different with SELI in both finite $\lambda$'s and vanishing $\lambda$'s cases (see page 25 of \cite{Christos22} for SELI closed-form $\bh^{\top} \bh$).



\end{document}